\documentclass[runningheads]{llncs}

\usepackage[mobile]{eccv}

\usepackage{tabularx}
\usepackage{xcolor}
\usepackage{pifont}

\usepackage{eccvabbrv}

\usepackage[table]{xcolor}

\definecolor{rankone}{RGB}{244,239,228} 
\definecolor{ranktwo}{RGB}{236,241,247} 

\newcommand{\topone}[1]{\cellcolor{rankone}\textbf{#1}}
\newcommand{\toptwo}[1]{\cellcolor{ranktwo}\underline{#1}}

\usepackage{amsmath}
\usepackage{graphicx}
\usepackage{booktabs}
\usepackage{multirow}
\usepackage{array}
\usepackage{algorithm}
\usepackage{algpseudocode} 

\newcolumntype{C}[1]{>{\centering\arraybackslash}p{#1}}

\usepackage[accsupp]{axessibility}  

\usepackage{hyperref}

\usepackage{orcidlink}

\DeclareMathOperator*{\argmin}{arg\,min}

\begin{document}

\title{MLP Splatting: Object-Centric Neural Fields}

\titlerunning{MLP-Splatting}

\newcommand{\cofirst}{\textsuperscript{*}}

\author{Shinjeong Kim\cofirst\orcidlink{0000-0002-3199-4111} \and
Yuzhou Cheng\cofirst\orcidlink{0009-0009-9463-0209} \and
Xin Kong\cofirst\orcidlink{0000-0003-4491-2522} \and\\
Paul H. J. Kelly\orcidlink{0000-0001-5905-1804} \and
Andrew J. Davison\orcidlink{0000-0002-3784-099X}}

\authorrunning{S.~Kim et al.}

\institute{Department of Computing, Imperial College London\\
\email{\{s.kim, y.cheng25, x.kong21, p.kelly, a.davison\}@imperial.ac.uk}\\
}

\maketitle

\begingroup
\renewcommand{\thefootnote}{*}
\footnotetext{Equal contribution.}
\endgroup

\begin{abstract}

3D representations are fundamental for scene rendering, understanding, and interaction. Approaches like Neural Radiation Fields and 3D Gaussian Splatting achieve impressive photorealistic novel-view synthesis, but lack the ability to easily decompose scene elements into object level, requiring additional segmentation or grouping. We present MLP-Splatting, a method that enables scene decomposition via a few expressive light-field primitives while providing photorealistic novel-view synthesis.

MLP-Splatting models each primitive as an independent compact MLP with localized spatial support that predicts radiance and opacity. In contrast to low-level Gaussian primitives or a single global radiance field, our neural primitives provide greater expressive capacity while remaining spatially localized. Rendering is performed through efficient sparse volumetric compositing over ray–primitive interactions.

Our primitives are supervised using RGB supervision alone, which yields primitives that represent local scene regions often corresponding to objects or object parts, enabling interactive object-level editing without segmentation masks by selecting a handful of primitives. Our method, augmented with optional semantic feature distillation, enables open-vocabulary scene interaction and open-set instant segmentation. Compared to state-of-the-art semantic 3DGS methods, we achieve substantially lower memory usage (1/7$\times$) and faster rendering (5$\times$). 

Project Page: \url{https://shinjeongkim.com/mlp-splatting}

\end{abstract}



%

\section{Introduction}
Reconstructing a 3D scene remains a fundamental and enduring problem in computer vision, closely tied to the scene representation. Different application goals have led to diverse primitives, ranging from discrete point clouds featuring ease of obtaining and processing, point-based primitives such as 3D Gaussian Splatting (3DGS)~\cite{Kerbl2023GaussianSplatting} for editability, mesh representations~\cite{lorensen1998marching} tailored for watertight modeling and physical simulation, and neural radiance fields (NeRF)~\cite{Mildenhall2020NeRF} for photorealistic novel-view synthesis.

\begin{figure}[t]
    \centering
    \begin{subfigure}[t]{0.68\linewidth}
        \centering
        \includegraphics[width=\linewidth]{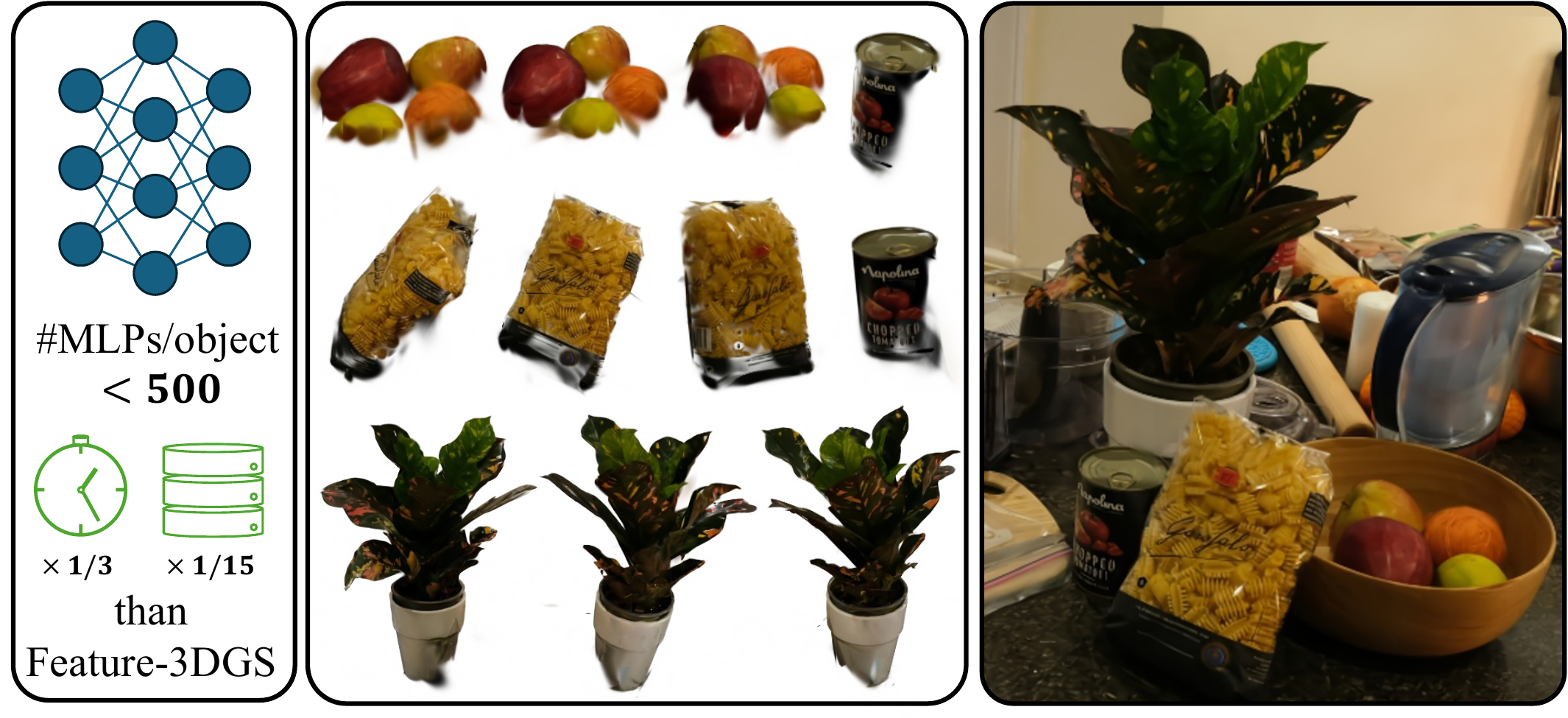}
        \caption{}
        \label{fig:teaser}
    \end{subfigure}
    \hfill
    \begin{subfigure}[t]{0.3\linewidth}
        \centering
        \includegraphics[width=\linewidth]{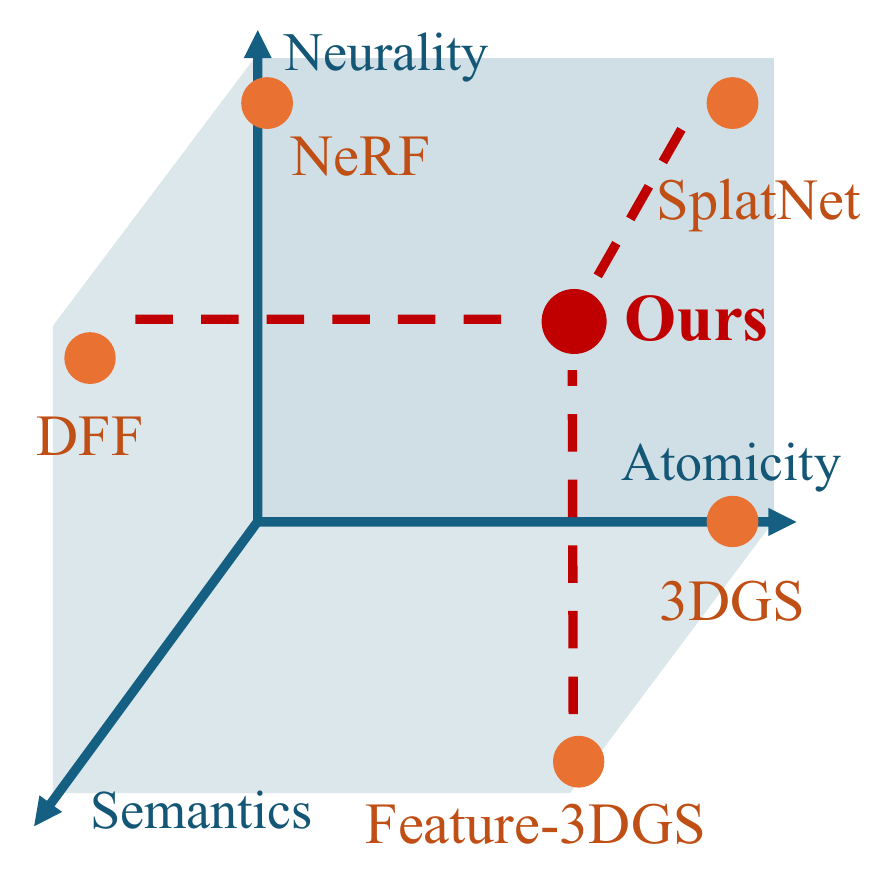}
        \caption{}
        \label{fig:3axis}
    \end{subfigure}
    
    \caption{MLP-Splatting Overview. \textbf{(a)} Example group of learned primitives corresponding to object or  object parts, and the rendered scene. \textbf{(b)} Conceptual comparison of 3D scene representations along neurality, atomicity, and semantic awareness. Our method combines the expressiveness of neural fields with the atomicity of primitives, enabling representations that better align with semantic structure.}
    \label{fig:overview}
    \vspace{-0.4cm}
\end{figure}

From another perspective, scene representations differ fundamentally in the level of abstraction of their primitives. Although many existing representations operate at the level of low-level geometric or radiometric elements, a natural question arises: should higher-level entities, such as objects or object parts, serve as the basic functional units of the representation?

In this work, we revisit this question from an object-level perspective. Most existing scene representations are not inherently object-centric: their basic primitives correspond to low-level geometric~\cite{lorensen1998marching, govindarajan2025radiant, mai2025ever, cheng2025logs, Kerbl2023GaussianSplatting} or radiometric elements~\cite{Mildenhall2020NeRF, Mueller2022InstantNGP}, while objects emerge only as derived groupings or semantic overlays. By object-centric, we mean that the representation's fundamental functional units correspond to objects or meaningful parts of objects. 
Although recent advances have begun to incorporate object awareness into neural scene representations, this awareness is typically introduced by distilling semantic knowledge onto pre-existing representations, such as implicit radiance fields in NeRF-based methods~\cite{kobayashi2022decomposing, zhi2021place, Kundu2022PanopticNeuralFields, kim2024garfield} or 3DGS~\cite{Ye2024GaussianGrouping, zhou2024feature, cen2025tackling}. In both cases, semantics are imposed on representations originally designed for photometric realism rather than arising from representations whose primitive units are themselves object- or part-level entities.

Yet we argue that the structure of real scenes itself naturally gives rise to object- and part-level organization, and that a scene representation can be designed to capture this structure directly. Many object-level properties in 3D space, such as spatial occupancy, color, and density exhibit strong coherence within an object, while sharp transitions are concentrated primarily near object boundaries, despite potentially high intrinsic complexity. We therefore introduce MLP-Splatting, which assigns independent MLP-based primitives to spatially significant parts of objects. These primitives collectively form a scene-level radiance field for rendering. Each primitive is continuous and differentiable, shares no parameters with the others, and in principle has unbounded expressive capacity, making it well-suited to modeling coherent object-level structure.

To showcase this object-centric property more clearly, we optionally augment the representation with semantic guidance through feature distillation~\cite{kobayashi2022decomposing, zhou2024feature}. This allows the learned primitives and their associated embeddings to support tasks such as segmentation and language-guided editing. With this supervision, MLP-Splatting achieves semantic metrics comparable to Feature-3DGS~\cite{zhou2024feature}, despite exhibiting orders-of-magnitude differences in model memory due to the nature of the representation. We further analyze the theoretical memory bound required to reach the same photometric error.

In summary, our contributions are fourfold:
\begin{enumerate}
\item \textit{Object-centric neural scene representation.}
We introduce \emph{MLP-Splatting}, a neural scene representation in which each primitive is modeled as an independent perceptron. This formulation treats objects or parts of objects as the fundamental functional units of the representation.

\item \textit{Emergent object-level structure.}
Under pure RGB supervision, primitives naturally specialize to coherent scene regions, enabling object-level editing \textbf{without segmentation}. With feature distillation, we achieve state-of-the-art rendering quality and competitive semantics in the Replica~\cite{straub2019replica} and ScanNet~\cite{Dai_2017_CVPR} datasets.

\item \textit{Efficient neural-primitive rendering.}
Our scene paradigm and efficient tile-based compositing pipeline enable real-time rendering at \textbf{$\sim$10} FPS (vs.~\textbf{$\sim$2} FPS in~\cite{zhou2024feature}) with only \textbf{120 MB} of memory (vs.~\textbf{880 MB} in~\cite{zhou2024feature}) on average.

\item \textit{Theoretical capacity analysis.}
We derive the \emph{Equal-Quality Memory Law}, showing that achieving the same pixel error requires \textbf{$\Omega(\varepsilon^{-2})$} parameters for Gaussian splatting but only \textbf{$\Theta(\varepsilon^{-3/s})$} for MLP-Splatting, where $s$ is the local smoothness order of the scenes.

\end{enumerate}

\section{Related Work}

\subsection{3D Scene Representations} 
NeRFs~\cite{Mildenhall2020NeRF} represent a scene as a continuous function that maps 3D positions and viewing directions to densities and radiance, optimized by differentiating the volumetric rendering of posed images. A large body of work has since improved the fidelity--efficiency trade-off~\cite{Barron_2021_ICCV, Reiser_2021_ICCV, Fridovich-Keil_2022_CVPR, Mueller2022InstantNGP, Barron2023ZipNeRF} and broadened the operating regime~\cite{Martin-Brualla_2021_CVPR, Barron2022MipNeRF360}. Instant-NGP~\cite{Mueller2022InstantNGP, taher2024fit} introduces multiresolution hash encodings and fused CUDA kernels to accelerate training and inference drastically. For large, unbounded outdoor captures, mip-NeRF~360~\cite{Barron2022MipNeRF360} addresses aliasing and scale imbalance via cone-based integration and scene parameterization, while Zip-NeRF~\cite{Barron2023ZipNeRF} further improves anti-aliasing by combining rendering/signal-processing ideas with grid-based representations.

In parallel, methods based on optimizable explicit primitives, such as 3D/2D Gaussian Splatting~\cite{Kerbl2023GaussianSplatting, huang20242d}, ellipsoid rendering~\cite{mai2025ever, nguyensplat}, polyhedral element rendering~\cite{govindarajan2025radiant, lutzow2025linprim, fry2026triangle}, etc., replace dense ray marching on a volumetric neural field with a fast visibility-aware rasterization or efficient closed-form volumetric rendering on a set of independently-optimizable unit geometric primitives, achieving high-quality real-time rendering. 
Our work seeks to combine the representational flexibility of neural fields with the computational efficiency of splatting-style pipelines, while pursuing a distinct representational goal: making object- and part-level structure emerge naturally from the scene representation.

\subsection{Object-level Scene Representations} 
A growing line of research seeks representations that explicitly expose objects for editing, tracking, and compositional reasoning. iLabel~\cite{zhi2022ilabel} revealed that a single MLP trained for scene representation automatically has a certain level of object awareness, such that 2D semantic masks can be predicted by a very small number of interactive annotations.
Various works have explored object-centric scene representations that model a scene as a composition of per-object components~\cite{Guo2020ObjectCentric,Yang2021ObjectCompositionalNeRF,kong2023vmap,Niemeyer_2021_CVPR,Wu_2023_ICCV}. Such decomposition enables object-level manipulations such as re-arrangement, relighting, editable rendering, and tracking~\cite{Guo2020ObjectCentric}.

Within radiance-field methods, object-compositional formulations introduce an explicit object branch, often conditioned on learnable codes, to support editable scene rendering~\cite{Yang2021ObjectCompositionalNeRF}, while real-time object-level mapping methods model each instance with a separate compact neural field to build object-level neural 3D scene representations online~\cite{kong2023vmap}.
In Panoptic Neural Fields~\cite{Kundu2022PanopticNeuralFields}, dynamic scenes are decomposed into object instances, each of which is represented by an MLP, and background regions, which are modeled separately. \cite{Kundu2022PanopticNeuralFields} relies on external pose estimation, tracking, and 2D panoptic predictions as pseudo-supervision.

For 3DGS, most extensions remain Gaussian-centric, with object-level structure introduced through grouping or auxiliary embeddings attached to individual Gaussians. Gaussian Grouping~\cite{Ye2024GaussianGrouping}, for example, learns per-Gaussian-grouping features guided by multiview mask predictions from a promptable segmentation model such as SAM~\cite{Kirillov2023SAM}, allowing open-world 3D segmentation and editing. These works demonstrate the value of object awareness; however, object-level entities are typically recovered through supervision, clustering, or auxiliary feature fields rather than serving as intrinsic functional primitives. By contrast, our approach explicitly treats objects and parts as independent functional units, each learned by a compact MLP, and composes these local models to render the scene.

\subsection{Open World Semantic Understanding}
Open-world semantics aims to recognize and localize concepts beyond a fixed label set, often by leveraging vision and vision-language foundation models. CLIP~\cite{Radford2021CLIP} provides a shared vision-language embedding space that supports zero-shot recognition and retrieval via text prompts. Promptable segmentation models such as SAM~\cite{Kirillov2023SAM} enable class-agnostic mask generation at scale, while open-set object detectors such as Grounding~DINO~\cite{liu2024grounding} perform phrase grounding and text-conditioned detection by integrating language into transformer-based detection. Prompt-based segmentation methods such as CLIPSeg~\cite{Luddecke2022CLIPSeg} further bridge text or image prompts to pixel-level masks.

Recent work has brought open-vocabulary, open-set, and prompt-driven semantics into 3D and neural rendering by lifting foundation-model signals into 3D-consistent fields. OpenScene~\cite{Peng2023OpenScene} predicts dense 3D features aligned with text and image embeddings to support open-vocabulary 3D understanding. LERF~\cite{Kerr2023LERF} embeds language features into radiance fields by supervising multi-view-consistent CLIP embeddings along rays, enabling open-ended language queries in 3D. GAR-Field~\cite{kim2024garfield} learns hierarchical 3D groups from SAM masks for open-ended and interactive scene-level grouping. In the Gaussian-splatting family, LangSplat~\cite{Qin2024LangSplat} constructs a 3D language field over Gaussians to support efficient open-vocabulary querying. Unlike these methods, we represent scenes with independent MLPs, each designed to represent an object or part of it, so semantics attach directly to the representation's native granularity.

\section{MLP-Splatting}
\begin{figure}[t]
    \centering
    \includegraphics[width=1\linewidth]{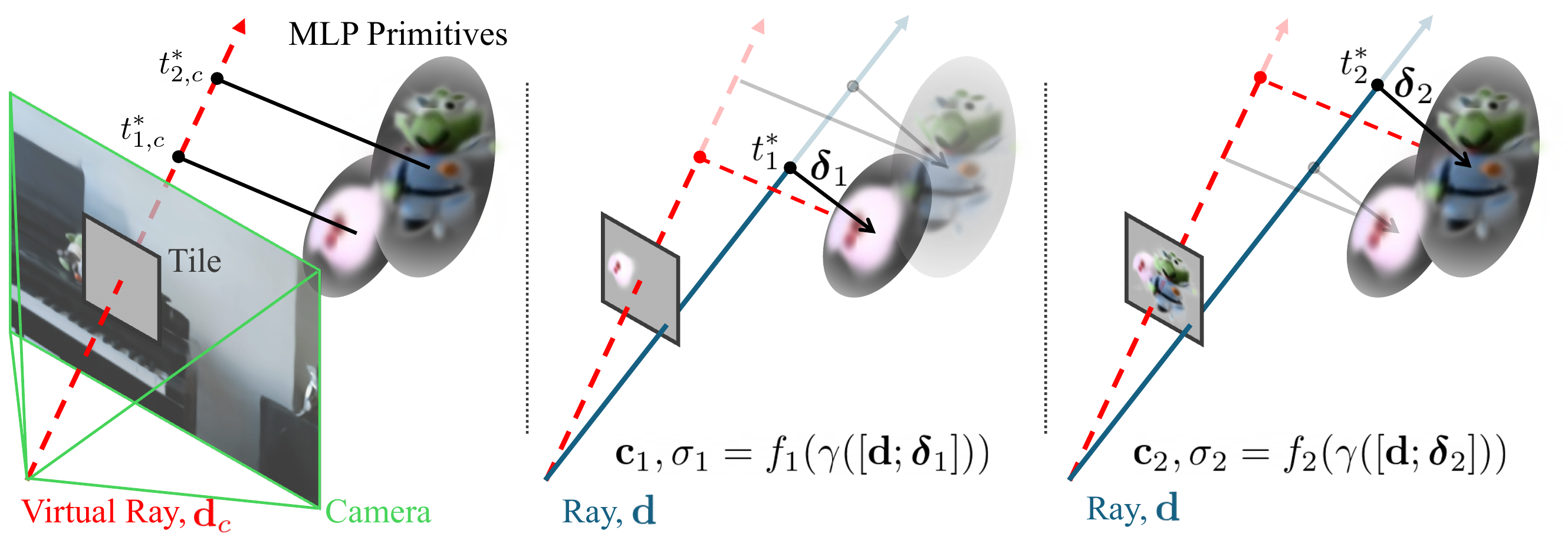}
    \caption{Procedure of MLP-Splatting pipeline. \textbf{(Left)} For each tile, sorting for MLPs is performed per tile, based on the depth calculated with respect to a virtual ray passing through the center of the tile; only MLPs for which the support region is splatted to the image plane that overlap with the tile are considered. \textbf{(Center)} Based on the sorted order, we perform the MLP weight loading and forward pass cooperatively per GPU thread block responsible for every pixel in the tile to obtain the color and density. \textbf{(Right)} Once finishing the MLP forward pass, the tile's threads proceed to alpha blend and move to the next MLP.}
    \label{fig:method}
    \vspace{-0.5cm}
\end{figure}

We represent a scene as a collection of independent primitives, each modeled by a compact MLP that predicts radiance and opacity within a local region. 
Altogether, these primitives form a piecewise-continuous approximation of the light-field. 
Given a camera ray, rendering reduces to evaluating the interaction between the ray and primitives in the vicinity, followed by standard volumetric compositing.

\subsection{Neural Primitive Representation}

Each primitive $i$ is parameterized by a center $\boldsymbol{\mu}_i \in \mathbb{R}^3$, a scale vector $\mathbf{s}_i \in \mathbb{R}^3$, and a rotation represented by a unit quaternion $\mathbf{q}_i \in S^3$, where $S^3$ denotes the unit 3-sphere in $\mathbb{R}^4$.
The scale and rotation together define an anisotropic Gaussian support region that models a soft spatial boundary in $\mathbb{R}^3$.
Associated with each primitive is an independent neural function 
$f_i(\cdot;\theta_i)$ mapping a 6D input to $(\mathbf{c}_i, \sigma_i)$, where 
$\mathbf{c}_i \in [0,1]^3$ denotes RGB radiance and 
$\sigma_i \ge 0$ denotes opacity density. 
No parameters are shared across primitives, allowing discontinuities between objects while preserving intra-object continuity.

\subsection{Ray Modeling and Interaction with Primitives}

Let a camera ray be defined as 
$\mathbf{r}(t) = \mathbf{o} + t \mathbf{d}$, 
where $\mathbf{o}$ is the origin and $\mathbf{d}$ is the normalized direction.
For the primitive $i$, we compute an \emph{optimal soft contact point} by minimizing a Gaussian-weighted distance from the ray to the primitive center:
\begin{equation}
t_i^\star 
= \argmin_{t\ge 0}\;
\big\|\mathbf{o}+t\mathbf{d}-\boldsymbol{\mu}_i\big\|_{\mathbf{\Lambda}_i}^2 \label{eq:tsolve}
~,
\end{equation}
where $
\mathbf{\Lambda}_i=\mathbf{R}(\mathbf{q}_i)\,\mathrm{diag}({s_{ix}^{2}}, {s_{iy}^{2}}, {s_{iz}^{2}})^{-1}\,\mathbf{R}(\mathbf{q}_i)^\top
\in \mathbb{R}^{3\times 3}$ ($\mathbf{R}(\cdot)$ is a rotation matrix of a given quaternion) is a positive definite matrix induced by the primitive's anisotropic Gaussian support.
This objective admits a closed-form solution:
\begin{equation}
t_i^\star \;=\;\max\!\left(0,\;\frac{\mathbf{d}^\top \mathbf{\Lambda}_i(\boldsymbol{\mu}_i-\mathbf{o})}{\mathbf{d}^\top \mathbf{\Lambda}_i \mathbf{d}}\right),
\qquad
\mathbf{p}_i^\star \;=\; \mathbf{o}+t_i^\star \mathbf{d}.\label{eq:intersection}
\end{equation}
We define the residual vector from the soft contact point to the primitive center as:
\begin{equation}
\boldsymbol{\delta}_i \;=\; \boldsymbol{\mu}_i - \mathbf{p}_i^\star,
\end{equation}
which, together with the ray direction, establishes a one-to-one mapping between viewing configuration and primitive-local spatial interaction.
The 6D geometric input to the primitive MLP is constructed as:
\begin{equation}
\mathbf{x}_i = [\mathbf{d} \, ; \, \boldsymbol{\delta}_i] \in \mathbb{R}^6,
\end{equation}
coupling directional dependence (viewing direction) with local geometric context at the contact location. 
This directly connects our formulation to the light-field parameterization, in which radiance depends jointly on viewing direction and spatial location, expressed here in primitive-centered coordinates.

To enhance high-frequency expressivity, we apply multi-resolution positional encoding to each scalar component of $\mathbf{x}_i$. 
For a scalar $p$, the encoding is defined as:
\begin{equation}
\gamma(p) =
\left(
\sin(2^0 \pi p), \cos(2^0 \pi p), \dots,
\sin(2^{L-1} \pi p), \cos(2^{L-1} \pi p)
\right),
\end{equation}
and extended component-wise to vectors. 
The encoded feature $\gamma(\mathbf{x}_i)$ is then used as the actual input to the primitive MLP.

The primitive predicts raw outputs
\(
(\tilde{\mathbf{c}}_i, \tilde{\sigma}_i) = f_i(\gamma(\mathbf{x}_i)),
\)
which are mapped to valid ranges via
\(
\mathbf{c}_i = \mathrm{sigmoid}(\tilde{\mathbf{c}}_i)
 \)
and
\(
\sigma_i = \mathrm{softplus}(\tilde{\sigma}_i).
\)
To enforce spatial locality, opacity is further modulated by a Gaussian falloff induced by the primitive’s anisotropic support.

\subsection{Rendering and Compositing}

Sorting by \cref{eq:intersection} leads to a geometrically correct order of primitives, requiring the thread responsible for each pixel to independently sort all primitives whose support region splatted to the image plane intersect the ray for the pixel. As an alternative for efficiency (see Fig.\ref{fig:method}), we can sort with a $t_{i, c}^*$ obtained by solving \cref{eq:tsolve} with $\mathbf{d}_c$, which is a virtual ray direction passing through the center of the tile the pixel belongs to. The pixel plane is divided into square tiles, as in \cite{Kerbl2023GaussianSplatting}.

For a ray intersecting $N$ primitives sorted by increasing $t_i^\star$, 
the rendered color is computed by front-to-back compositing:
\begin{equation}
\mathbf{C}
=
\sum_{i=1}^{N}
T_i \alpha_i \mathbf{c}_i,
\end{equation}
with
\(
\alpha_i = 1 - \exp(-\sigma_i \ell_i),
~
T_i = \exp\!\left(-\sum_{j=1}^{i-1} \sigma_j \ell_j\right),
\)
where \(\ell_i\) denotes the effective integration length associated with
the \(i\)-th ray--primitive interaction.

This formulation resembles not only alpha blending of 3DGS, but also the discrete approximation of the volume rendering integral used in NeRF, while replacing dense sampling with sparse per-primitive interactions.

To enable efficient full-resolution rendering, primitives are rasterized into screen-space tiles using conservative overlap tests based on their projected support regions. 
Within each tile, only spatially relevant primitives are evaluated, allowing scalable rendering without global ray marching.

\subsection{Feature Embedding}\label{sec:feature_embedding}
We attach semantic features to each primitive using the same ray--primitive interaction used for radiance.
Let $\boldsymbol{\delta}_i=\boldsymbol{\mu}_i-\mathbf{p}_i^\star$ and
$\mathbf{u}_i=\mathrm{clip}_{[-1,1]^3}\!\big(\mathbf{R}(\mathbf{q}_i)^\top\boldsymbol{\delta}_i \oslash \mathbf{s}_i\big)$
be the scale-normalized local offset.
This local offset enables regions with similar appearance but different local orientation to be separated in the semantic space, as follows:
\begin{equation}
g_{ik}=\mathrm{softmax}_k\!\big(\mathbf{a}_k^\top \mathbf{u}_i\big),\qquad
\mathbf{f}_i=\sum_{k=1}^{4} g_{ik}\mathbf{e}_{ik},
\end{equation}
where $\mathbf{e}_{ik}\in\mathbb{R}^{C_f}$ are learnable slot embeddings, and $\mathbf{a}_{k}$ are directional conditioning defined by a canonical tetrahedral frame
$\mathcal{A}=\{\mathbf a_k\}_{k=1}^4\subset\mathbb S^2$.
This is the minimal affinely independent basis in 3D and yields an isotropic angular partition. This per-primitive feature embedding is rendered with the same volumetric weights as color:
$\mathbf{F}=\sum_{i=1}^{N}T_i\alpha_i\,\mathbf{f}_i$.

\subsection{Optimization}\label{sec:optimization}

Given ground-truth image $I$ with color channels $\mathbf{C}^\ast$, we supervise rendered colors $\mathbf{C}$ using a photometric objective,
\begin{equation}
\mathcal{L}_{rgb}
=
(1-\lambda)\,\|\mathbf{C}-\mathbf{C}^\ast\|_1
+
\lambda\,\big(1-\mathrm{SSIM}(\mathbf{C},\mathbf{C}^\ast)\big),
\end{equation}
which balances per-pixel fidelity and structural similarity \cite{Kerbl2023GaussianSplatting}. When jointly training semantics, we align the rendered feature map with the latent embedding from a vision foundation model:
\begin{equation}
\mathcal{L}
= \mathcal{L}_{rgb} + \mathcal{L}_{f},  \qquad
\mathcal{L}_{f}
=
\left\|\mathbf{F}-\mathbf{F}_{t}\right\|_{1}, 
\end{equation}
where $\mathbf{F}_{t}$ denotes the teacher latent embedding of image $I$.

\begin{figure}[t!]
    \centering
    \includegraphics[width=1\linewidth]{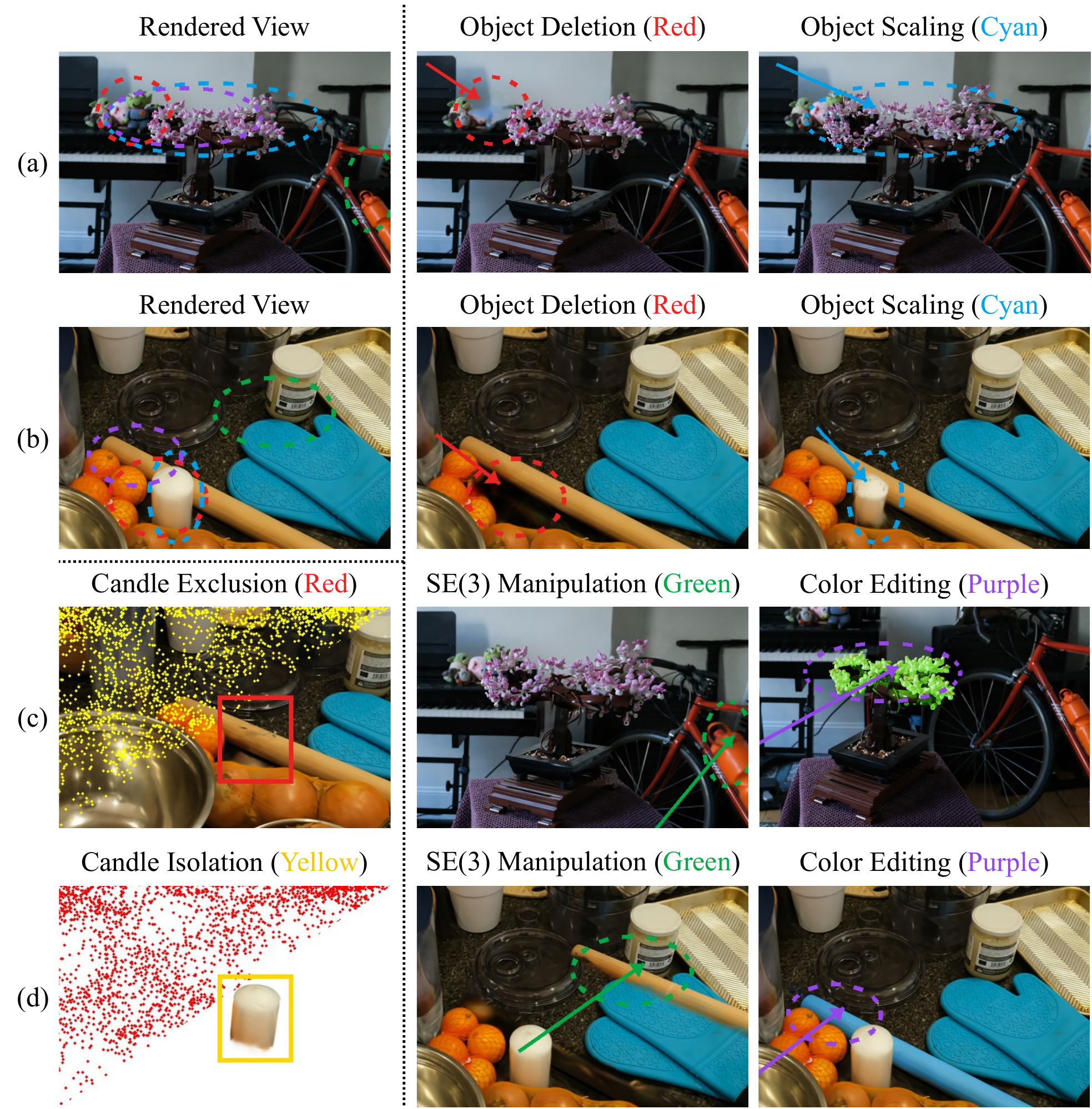}
    \caption{Interactive scene editing demonstrations. Row (a) and (b) show object scaling and deletion demonstrated on a bonsai tree and a candle. Leftmost column of row (c) and (d) shows exclusion and isolation of the candle. Rest parts of the row (c) and (d) show color editing and rigid body transformations applied to the bonsai tree's and the rolling pin's primitives, and how they are rendered from the new location.}
    \label{fig:rgb_edit}
    \vspace{-0.5cm}
\end{figure}

\section{Experiments}

\subsection{Implementation Details}
The 6D input is positional encoded into a 72-dimensional feature vector. Each MLP architecture is $72$–$32$–$32$–$4$, producing RGB color and raw density. The feature embedding has a dimension of 512, and its rendered feature is directly guided by 2D semantic features.
We implement a custom tile-based CUDA renderer with a $8\times8$ tile size, which is critical for efficient rendering. For each tile, we first perform an AABB--tile-frustum intersection test to identify candidate primitives, and then apply a block-level radix sort using their depths along a representative virtual ray passing through the tile center. This virtual ray is used only to determine the primitive ordering within the tile and is independent of actual per-pixel rendering rays.

We initialize the MLP primitives using an octree-based sampling of the COLMAP~\cite{schoenberger2016sfm} point cloud, ensuring sufficient spatial coverage and geometric fidelity. During training, the primitives in the vicinity naturally compete with each other: some specialize in representing parts of objects, while others gradually vanish as their scale and opacity approach zero. Exploiting this property to keep the proposed method simple to analyze, no explicit density control mechanism is employed. 

Optimization is done via the Adam optimizer~\cite{zhang2018improved} with learning rates of $10^{-2}$ for log-scales, $10^{-4}$ for positions, and $10^{-3}$ for all other parameters. 
The model is trained for 60{,}000 iterations with a single NVIDIA RTX 6000 Ada GPU.
All evaluations are conducted on an NVIDIA RTX 4090 GPU, following the evaluation setting of~\cite{zhou2024feature}.

\paragraph{Evaluation metrics.}
We report two groups of metrics following the Replica protocol of Feature-3DGS~\cite{zhou2024feature}.
\textbf{Photometric quality} is evaluated on held-out novel views using PSNR, SSIM, and LPIPS \cite{Kerbl2023GaussianSplatting}.
\textbf{Semantic quality} is evaluated by rendering per-pixel semantic predictions and computing mean intersection-over-union (mIoU) and pixel accuracy:
$\mathrm{IoU}_c=\frac{\mathrm{TP}_c}{\mathrm{TP}_c+\mathrm{FP}_c+\mathrm{FN}_c}$,
$\mathrm{mIoU}=\frac{1}{|\mathcal C|}\sum_{c\in\mathcal C}\mathrm{IoU}_c$,
and $\mathrm{Acc}=\frac{1}{|\Omega|}\sum_{p\in\Omega}\mathbf{1}\!\left[\hat y_p=y_p\right]$,
where $\mathcal C$ is the evaluated label set and $\Omega$ denotes image pixels.

\paragraph{Baselines.}
We compare against (i) NeRF-DFF~\cite{kobayashi2022decomposing}, which distills CLIP-LSeg features into a NeRF-style feature field, and
(ii) Feature-3DGS~\cite{zhou2024feature}, which distills the same teacher features into 3D Gaussian primitives with an efficient rasterization pipeline.
Both baselines are evaluated on the Replica~\cite{straub2019replica} dataset, and Feature-3DGS is evaluated on the ScanNet~\cite{Dai_2017_CVPR} dataset, where sequences are sampled following~\cite{zhou2024feature}.


\begin{table}[tb]
\caption{Novel-view synthesis evaluation on Replica and ScanNet datasets.}
\label{tab:fig1_quality}
\centering
\renewcommand{\arraystretch}{1.08}
\setlength{\tabcolsep}{4pt}
\begin{tabularx}{\textwidth}{@{}
>{\raggedright\arraybackslash}m{2.55cm}|
>{\hsize=1.2\hsize\centering\arraybackslash}X
>{\hsize=0.8\hsize\centering\arraybackslash}X
|
>{\hsize=1.2\hsize\centering\arraybackslash}X
>{\hsize=0.8\hsize\centering\arraybackslash}X
@{}}
\noalign{\hrule height 1.1pt}
Dataset & \multicolumn{2}{c|}{Replica} & \multicolumn{2}{c@{}}{ScanNet} \\
[-0.95ex]
\rule{2.5cm}{0.4pt}
& \multicolumn{1}{c}{\makebox[0pt]{\rule{2.5cm}{0.4pt}}}
& \multicolumn{1}{c|}{\makebox[0pt]{\rule{1.5cm}{0.4pt}}}
& \multicolumn{1}{c}{\hspace{0.12cm}\makebox[0pt]{\rule{2.5cm}{0.4pt}}}
& \multicolumn{1}{c@{}}{\makebox[0pt]{\rule{1.5cm}{0.4pt}}} \\
[-0.95ex]
Metric|Method & \mbox{Feature-3DGS} & Ours & \mbox{Feature-3DGS} & Ours \\
\hline
PSNR $\uparrow$    & 36.18 & \topone{36.25} & 23.32 & \topone{25.35} \\
SSIM $\uparrow$    & 0.964 & \topone{0.971} & 0.817 & \topone{0.830} \\
LPIPS $\downarrow$ & \topone{0.079} & 0.090 & \topone{0.362} & 0.403 \\
\noalign{\hrule height 1.1pt}
\end{tabularx}
\end{table}

\subsection{MLP as an Implicit Primitive for Explicit Representation}
Although the primary goal of MLP-Splatting is to construct an object-centric scene representation, it retains strong photometric rendering performance. 
Despite decomposing the scene into independent neural primitives, the resulting model achieves faithful view synthesis comparable to scene-centric approaches.

Statistics shown in \cref{tab:fig1_quality} illustrate that our formulation preserves high-frequency appearance details 
while maintaining spatial coherence across viewpoints. Through a comparison of the modeling behaviors of 3DGS and MLP-Splatting, as illustrated in the second column of \cref{fig:lseg_semantics}, we observe clear differences.
3DGS employs a large number of Gaussian ellipsoids to represent the fine structures of the blinds, yet the learned result ultimately collapses into blurry, coarse regions.
In contrast, MLP-Splatting successfully represents each slat with one or a few elongated MLP primitives, producing the correct striped structure. This demonstrates that modeling objects as independent functional units 
does not compromise photometric fidelity; rather, it offers a structurally aligned alternative to monolithic scene-wide fields.

We further analyze the joint training of appearance and semantics. In \cref{tab:mlp_3dgs_dualcol_top5}, we present a side-by-side comparison between MLP-Splatting and Feature-3DGS in terms of the parameter count required to represent the same object in the same scene. For each scene, we select the top-5 object groups (covering about $95\%$ of all primitives), and estimate their counts using a softmax-based threshold ($\tau=0,15$). We observe that MLP-Splatting typically requires fewer than $1\%$ of the Gaussian ellipsoids required by Feature-3DGS. A single MLP has fewer parameters than ten semantic Gaussian ellipsoids. This demonstrates an order-of-magnitude gap in representational capacity.

\begin{table}[tb]
\caption{Primitive counts per-scene for the five objects with the most primitives across scenes. MLP and 3DGS are shown side-by-side for direct comparison.}
\label{tab:mlp_3dgs_dualcol_top5}
\centering
\scriptsize
\begin{tabular*}{\columnwidth}{@{\extracolsep{\fill}}l*{7}{cc}}
\toprule
\multirow{2}{*}{Scene} &
\multicolumn{2}{c}{floor} &
\multicolumn{2}{c}{rug} &
\multicolumn{2}{c}{table} &
\multicolumn{2}{c}{chair} &
\multicolumn{2}{c}{bag} &
\multicolumn{2}{c}{wall} &
\multicolumn{2}{c}{Total} \\
\cmidrule(lr){2-3}\cmidrule(lr){4-5}\cmidrule(lr){6-7}\cmidrule(lr){8-9}\cmidrule(lr){10-11}\cmidrule(lr){12-13}\cmidrule(lr){14-15}
& MLP & 3DGS & MLP & 3DGS & MLP & 3DGS & MLP & 3DGS & MLP & 3DGS & MLP & 3DGS & MLP & 3DGS \\
\midrule
room0   & 253 & 48490 & 485 & 21621 & \phantom{1}511 & \phantom{1}98197 & 782 & 43149 & 76 & 17484 & -   & -      & 2107 & 228941 \\
room1   & \phantom{1}91 & 21525 & 116 & 32751 & 1961 & 332611 & -   & -     & 44 & 10389 & 428 & 140025 & 2640 & 537301 \\
office3 & 174 & 50671 & -   & -     & \phantom{1}198 & \phantom{1}24242 & 290 & 36867 & 49 & \phantom{1}9102 & 351 & \phantom{1}62102 & 1062 & 182984 \\
office4 & 171 & 29803 & -   & -     & \phantom{1}178 & \phantom{1}35771 & 203 & 49684 & 38 & \phantom{1}9425 & 296 & \phantom{1}51922 & \phantom{1}886 & 176605 \\
\bottomrule
\end{tabular*}
\vspace{-0.5cm}
\end{table}

\subsection{Photometry-Driven Interactive Scene Editing}
Notably, even when trained solely with RGB supervision, 
MLP-Splatting exhibits emergent ``object-level'' structure. 
Without any segmentation masks, language priors, or identity supervision, 
each primitive naturally specializes to represent a single object 
or a coherent part of an object. For example, in the lower-left of \cref{fig:rgb_edit}, the candle object is jointly represented by 11 MLP primitives. They are largely distributed along the edges of the cylindrical structure. By selecting these primitives and masking them out, the object can be removed from the scene. 

We attribute this emergent behavior to the inductive bias of our representation. Since primitives are independent functional predictors with localized spatial support, optimization favors configurations in which each MLP exclusively models a coherent and confined subset of the scene. As a result, semantic grouping naturally emerges from photometric reconstruction alone.

This emergent object-level decomposition enables interactive scene editing directly at the primitive level with ease. 
By manipulating or removing selected primitives, 
users can edit individual objects without retraining or auxiliary segmentation models. \cref{fig:rgb_edit} demonstrates object removal, recoloring, and partial editing achieved purely through primitive-level operations. 

The mathematical implementation is tightly coupled with the MLP primitives. Each primitive has an explicit spatial position and an orientation parameterized by a quaternion. 
Therefore, rigid-body manipulation can be implemented by applying an $\mathrm{SE}(3)$ transformation to the poses of a single primitive or a selected group of primitives, updating their positions and orientations consistently.
Furthermore, object-level scaling can be achieved by modifying the scale parameters of primitives.

At a deeper level, we can directly operate on the MLP's parameters. In particular, global color and opacity adjustments can be implemented by modifying the bias of the final layer. 
The color transformations shown in \cref{fig:rgb_edit} are obtained by adding small offsets along target color directions in the last bias vectors of the MLPs.
More fine-grained texture manipulation can be achieved by post-optimizing the internal MLP parameters, similar to DFF~\cite{kobayashi2022decomposing}.

\subsection{Multi-view Consistent 3D Segmentation}

As discussed in the previous section, the proposed MLP primitives enable easy editing on objects, stemming from the property that each learned radiance field is typically constrained to object boundaries. To better demonstrate this characteristic, arising from the combination of both explicit and implicit scene representation properties, we augment our primitives with semantic embeddings and conduct language-guided editing and segment-anything experiments in the style of \cite{zhou2024feature}. Through experiments in this section, we show that MLPs can still effectively preserve fine-grained semantic details despite covering volumes that are significantly larger than those of conventional primitives.

\subsubsection{Language-guided Segmentation}
\begin{figure}[t]
    \centering
    \includegraphics[width=1\linewidth]{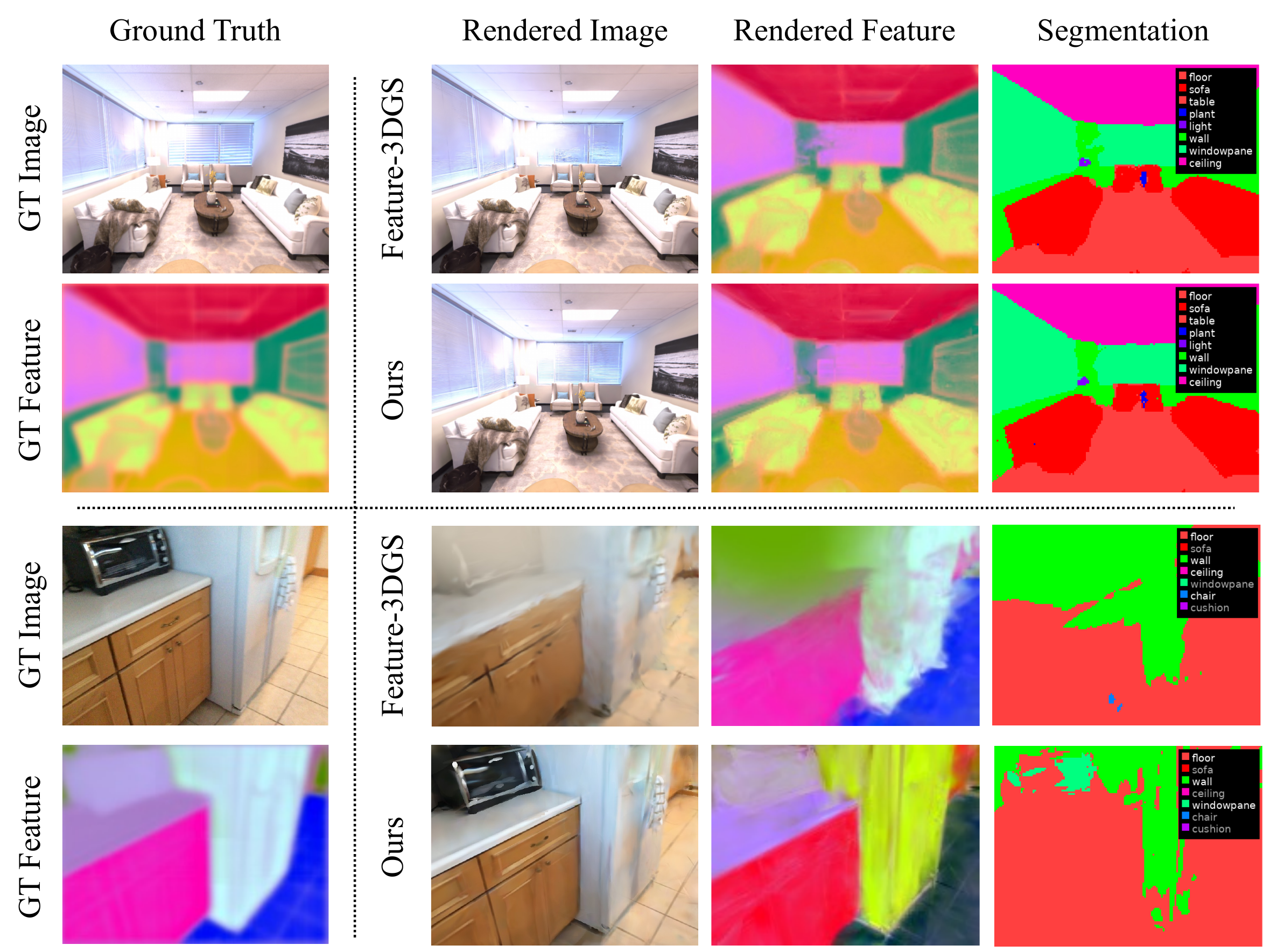}
    \caption{3D semantic segmentation with LSeg-guided embeddings rendered on novel views. The top two rows are from the Replica~\cite{straub2019replica} dataset, and the bottom two rows are from the ScanNet~\cite{Dai_2017_CVPR} dataset.}
    \label{fig:lseg_semantics}
    \vspace{-0.5cm}
\end{figure}




We follow the approach of \cite{zhou2024feature}, which distills the feature map of image encoder of LSeg-CLIP into embeddings that are independent for each primitive, and learns to represent semantic feature field. These independent embeddings are trained jointly with our primitives under supervision from the feature map and images.

As shown in~\cref{tab:fig1_semantic_efficiency}, evaluated under the protocol proposed in~\cite{zhou2024feature, kobayashi2022decomposing}, our method outperforms~\cite{kobayashi2022decomposing} across all semantic segmentation metrics, despite each primitive modeling a light-field-like radiance field, and performs competitively with Feature 3DGS. This strong performance reflects the explicit representation property of our approach: by decomposing the scene into multiple primitives that partition space, our representation encourages semantic regions to better align with object boundaries.

At the same time, our method is significantly more efficient than~\cite{zhou2024feature} in both memory usage and rendering speed, demonstrating the high compression capability of neural scene representations. The speed improvement arises because this high compression allows the scene to be represented with fewer primitives, thereby effectively avoiding the large alpha-blending overhead of numerous semantic embeddings encountered by~\cite{zhou2024feature}.

Furthermore, as illustrated in \cref{fig:lseg_semantics}, our method not only is quantitatively performant, but also qualitatively produces semantic segmentation results that are comparable to~\cite{zhou2024feature}.

\subsubsection{Segment Anything}
\begin{figure}[t]
    \centering
    \includegraphics[width=1\linewidth]{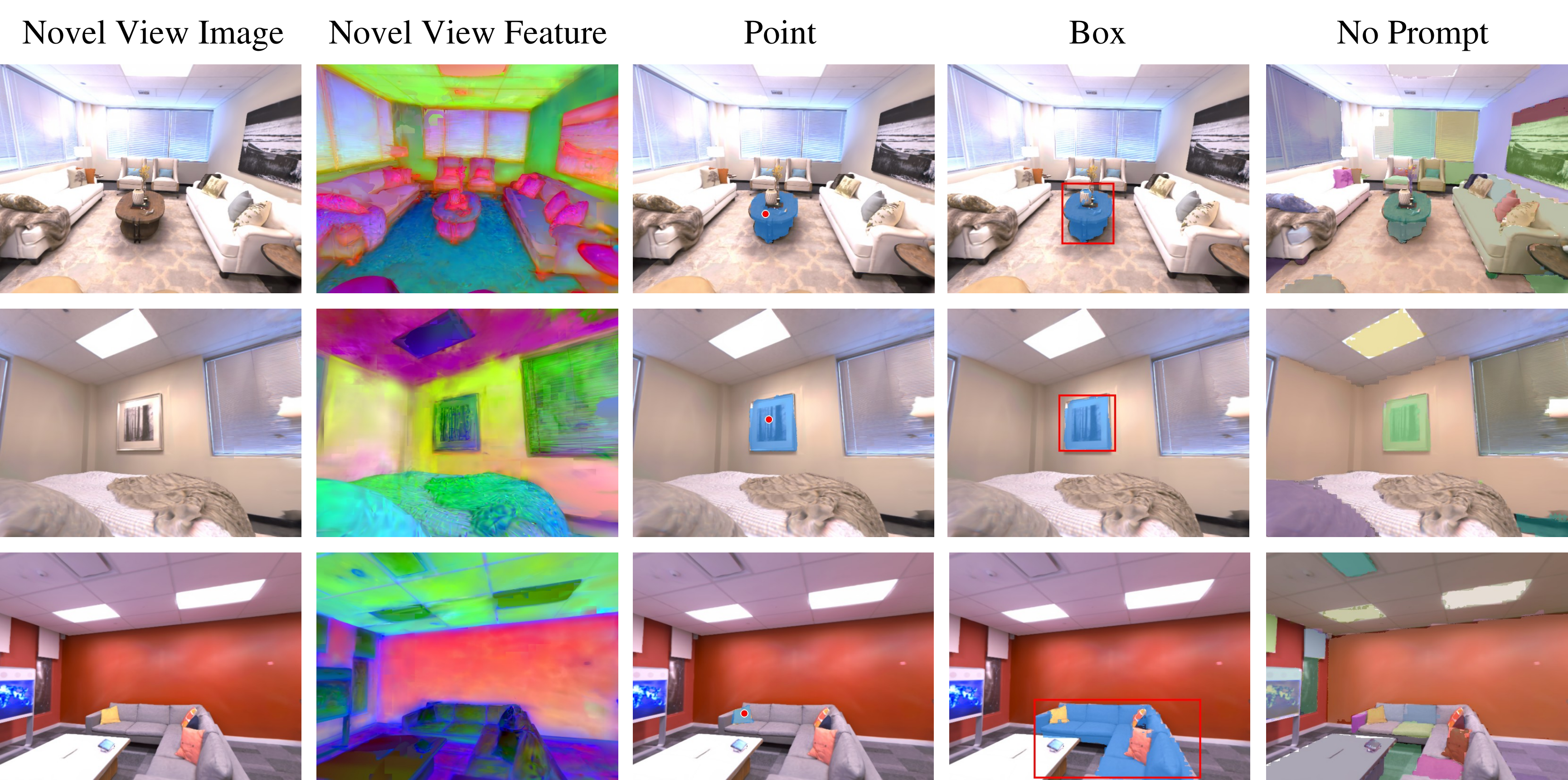}
    \caption{Novel view semantic segmentation with SAM embeddings.}
    \label{fig:sam_semantics}
\end{figure}


In \cref{fig:sam_semantics}, we perform instance segmentation across multiple views and scenes using our primitives and embeddings trained with the embedding guidance of \cite{Kirillov2023SAM}, following the protocol of \cite{zhou2024feature}. For either the point or the box prompt, it is shown that the obtained instance segmentation has a clear boundary and covers most of our intended objects, demonstrating that the 2D feature field rendered from our per-primitive semantic embedding is not only well aligned with the SAM feature from which it is guided, but also multi-view consistent, given that \cref{fig:sam_semantics} is rendered from a novel viewpoint.

\begin{table}[tb]
\caption{Semantic and efficiency comparison on Replica and ScanNet datasets.}
\label{tab:fig1_semantic_efficiency}
\centering
\renewcommand{\arraystretch}{1.08}
\setlength{\tabcolsep}{4pt}
\begin{tabularx}{\textwidth}{@{}
>{\raggedright\arraybackslash}m{2.55cm}|
>{\hsize=0.85\hsize\centering\arraybackslash}X
>{\hsize=1.2\hsize\centering\arraybackslash}X
>{\hsize=0.925\hsize\centering\arraybackslash}X
|
>{\hsize=1.25\hsize\centering\arraybackslash}X
>{\hsize=0.775\hsize\centering\arraybackslash}X
@{}}
\noalign{\hrule height 1.1pt}
Dataset & \multicolumn{3}{c|}{Replica} & \multicolumn{2}{c@{}}{ScanNet} \\
[-0.95ex]
\rule{2.5cm}{0.4pt}
& \multicolumn{1}{c}{\makebox[0pt]{\rule{1.6cm}{0.4pt}}}
& \multicolumn{1}{c}{\makebox[0pt]{\rule{2.1cm}{0.45pt}}}
& \multicolumn{1}{c|}{\makebox[0pt]{\rule{1.2cm}{0.4pt}}}
& \multicolumn{1}{c}{\hspace{0.12cm}\makebox[0pt]{\rule{2.1cm}{0.3pt}}}
& \multicolumn{1}{c@{}}{\makebox[0pt]{\rule{1.1cm}{0.4pt}}} \\
[-0.95ex]
Metric|Method & Nerf-DFF & \mbox{Feature-3DGS} & Ours & \mbox{Feature-3DGS} & Ours \\
\hline
mIoU (\%)$\uparrow$
& 63.6 & \topone{78.9} & \toptwo{77.7}
& 60.0 & \topone{64.3} \\
Accuracy (\%)$\uparrow$
& 86.4 & \topone{94.3} & \toptwo{93.9}
& 87.0 & \topone{89.4} \\
Memory (MB)$\downarrow$
& \toptwo{86.5} & 629.7 & \topone{32.0}
& 1120 & \topone{206} \\
FPS$\uparrow$
& $<3$ & \toptwo{$\sim$3} & \topone{$\sim$8}
& $<1$ & \topone{$\sim$12} \\ 
\noalign{\hrule height 1.1pt}
\end{tabularx}
\end{table}

\section{Discussion}\label{sec:discussion}
Lastly, we compare MLP-Splatting and Gaussian Splatting at the level of
\emph{density field approximation under volumetric rendering}. We found that Gaussian splatting is fundamentally limited by geometric surface coverage: analytic primitives cannot reproduce sharp discontinuities without allocating $\Theta(h^{-2})$ components along object boundaries.
In contrast, MLP-Splatting performs local functional approximation within smooth regions and preserves discontinuities via partition-of-unity, achieving optimal Sobolev rates \cite{yarotsky2017error,siegel2024sharp}.

Our analysis is concentrated into a theorem we call the Equal-Quality Memory Law. This theorem not only explains the bounds on the number of parameters but also indicates why our representation can naturally learn object-level entities, whereas Gaussian Slatting struggles. We briefly illustrate the mathematical formulation (\cref{sec:5_1}) and list important theorems (\cref{sec:5_2}) while leaving the proofs for the supplementary material. For ease of exposition, we do not use bold roman notation for vectors in this section.

\subsection{Function-Space Perspective}\label{sec:5_1}
 Let $\sigma : \mathbb R^3 \to \mathbb R_{\ge 0}$ denote the volume density and $c : \mathbb R^3 \times \mathbb S^2 \to \mathbb R^3$ the color field. Given a ray $r(t)=o+td$, volumetric rendering produces a pixel value $
\mathcal R_r[\sigma,c]
=
\int_0^{L}
T_\sigma(t)\,
\sigma(r(t))\,
c(r(t),d)\,dt,
$ where $
T_\sigma(t)
=
\exp\!\left(-\int_0^t \sigma(r(s))\,ds\right)
$ is the accumulated transmittance.
Given a set of rays $\mathcal U$, we measure pixel error by 
\begin{equation}
\mathcal E_p((\sigma,c),(\tilde\sigma,\tilde c))
=
\left(
\frac{1}{|\mathcal U|}
\sum_{r\in\mathcal U}
\|
\mathcal R_r[\sigma,c]
-
\mathcal R_r[\tilde\sigma,\tilde c]
\|^p
\right)^{1/p}.
\end{equation}

\noindent
\textbf{Model classes.}
Gaussian splatting constructs density as a $K$-term Gaussian mixture  
$
\sigma_K(x)=\sum_{k=1}^K a_k 
\exp(-\tfrac12(x-\mu_k)^\top\Sigma_k^{-1}(x-\mu_k)),
$ with $x$ 3D position
whereas MLP-Splatting represents density as a sum of neural primitives
$
\sigma_N(x)=\sum_{i=1}^N \rho(f_i(x, \theta_i)),
$
with 6D input ($x$), width-$m$ networks ($f_i$) and $\rho$ enforcing nonnegativity.

The following shows how many parameters are required to approximate a piecewise-smooth target scene $(\sigma^*,c^*)$ up to pixel error $\varepsilon$.


\subsection{Main Theoretical Results}\label{sec:5_2}

\begin{theorem}[Lower Bound for Gaussian Splatting]
\label{thm:gs-main}
For any nonnegative $K$-Gaussian mixture $\sigma_K$,
\begin{equation}
\inf_{\sigma_K}
\mathcal E_1\big((\sigma^*,c^*),(\sigma_K,c^*)\big)
\;\gtrsim\;
C K^{-1/2}.
\end{equation}
\end{theorem}

The proof combines the analyticity of Gaussian mixtures,
a geometric tube argument near the discontinuity set,
and stability of the rendering operator.

\begin{theorem}[Upper Bound for MLP-Splatting]
\label{thm:mlp-main}
Assume $(\sigma^*,c^*)$ is piecewise $C^s$ with jump set $\Gamma$. There exists a partition-of-unity construction such that
\begin{equation}
\mathcal E_2\big((\sigma^*,c^*),(\sigma_N,c_N)\big)
\le
C\big(h^s+m^{-s/3}\big),
\end{equation}
where $h$ is patch size and $m$ network width.
Balancing the terms yields parameter cost
$\Theta(\varepsilon^{-3/s})$ to achieve pixel error $\varepsilon$.
\end{theorem}

The proof combines local neural approximation rates,
geometric partitioning along $\Gamma$, 
and stability of volumetric rendering.

\begin{corollary}[Equal-Quality Memory Law]
Given pixel error $\varepsilon$ in $d=3$ we aim to achieve, 
\begin{itemize}
\item Gaussian splatting requires $\Omega(\varepsilon^{-2})$ parameters;
\item MLP-Splatting admits a construction with $\Theta(\varepsilon^{-3/s})$ parameters.
\end{itemize}
For $s\ge 2$ (i.e., when the target scene is piecewise twice differentiable), $\varepsilon^{-3/s}\le \varepsilon^{-1.5}$, 
strictly improving the Gaussian rate.
\end{corollary}

While real-world scenes are globally non-smooth, they often exhibit smooth structure piecewise: discontinuities are concentrated near object boundaries, whereas the fields within each region remain locally regular. This suggests that, in practice, MLP-Splatting can achieve a comparable or lower rendering error with a comparable or smaller number of parameters than 3DGS.

\vspace{-0.1cm}
\section{Limitations}\label{sec:limitations}
We do not explicitly design a density control strategy for MLP-Splatting. The positional-gradient–based densification used in 3DGS does not transfer well, as the gradients are often weak and noisy. Moreover, MLP primitives do not expose an explicit opacity parameter for pruning.

However, under RGB supervision we observe that primitives in regions with consistent appearance tend to rapidly expand their scales, causing neighboring primitives to become effectively inactive, while primitives in appearance-complex regions exhibit parameter oscillation to fit details. Both behaviors can reduce efficiency or reconstruction quality, suggesting that a suitable density control mechanism could further improve the framework.

\vspace{-0.1cm}
\section{Conclusion}

We introduced \textit{MLP-Splatting}, a neural scene representation in which each primitive is modeled as an independent compact MLP. 
Unlike existing approaches built on geometric primitives, our formulation treats neural predictors themselves as the fundamental units of the scene representation.

This design encourages primitives to specialize to coherent scene regions, leading to emergent object-level structure under RGB supervision and enabling straightforward integration with 2D semantic feature distillation. We show that compact neural primitives can represent scenes with strong photometric fidelity, competitive semantics, and substantially improved efficiency. These results suggest an alternative paradigm for neural scene representations. We hope this work motivates a shift from attaching semantics to low-level primitives toward designing representations whose inductive bias makes object-level structure emerge as a first-class property.

\vspace{-0.1cm}
\section*{Acknowledgements}
This research was supported by the EPSRC grant EP/Y020499/1 and the ARIA seed project NACB-SE02-P05. We appreciate all members of the Robot Vision Group for insightful and valuable discussions. In particular, we would like to thank Eric and Riku for their constructive discussions that helped shape the initial ideas, and Marwan for his help in refining the manuscript.


%
%
\bibliographystyle{splncs04}
\bibliography{arxiv}

\clearpage
\appendix
\section{Ablation Study}
\subsection{Novel-view Synthesis without Semantic Guidance}
\begin{table}[h]
\caption{Effect of co-training on the Replica dataset.}
\label{tab:fig1_cotraining}
\centering
\begin{tabularx}{0.9\textwidth}{@{}l *{3}{>{\centering\arraybackslash}X}@{}}
\toprule
 Replica     & PSNR \(\uparrow\) & SSIM \(\uparrow\) & LPIPS \(\downarrow\) \\
\midrule
MLP-Splatting w/o co-train & 35.57 & 0.969 & 0.094 \\
MLP-Splatting w co-train   & \topone{36.25} & \topone{0.971} & \topone{0.090} \\
\bottomrule
\end{tabularx}
\end{table}

This subsection compares MLP-Splatting with and without semantic guidance. \cref{tab:fig1_cotraining} shows that jointly optimizing semantic embeddings together with photometric supervision improves the rendering quality. 
This behavior is non-trivial. In many cases, if the geometry learned from RGB supervision is inaccurate, introducing semantic supervision may alter the geometry in ways that improve semantic consistency but degrade photometric fidelity, or vice versa.

In contrast, we observe consistent improvements across all three metrics (PSNR $+$ 0.68\,dB, SSIM $+$ 0.002, and LPIPS $-$ 0.004). This indicates that even with purely RGB supervision, MLP-Splatting already learns a reasonably accurate object-level spatial occupancy of the scene. The additional semantic signal, therefore, acts as a regularizer that refines the representation without compromising photometric accuracy.

\subsection{Number of Feature Slots Ablation}

\begin{table}[htb]
\caption{Ablation on the number of feature slots, on the Replica dataset.}
\label{tab:slot_ablation}
\centering
\begin{tabularx}{0.8\textwidth}{@{}l *{3}{>{\centering\arraybackslash}X}@{}}
\toprule
    Replica  & Memory (MB) & FPS & mIoU (\%) \\
\midrule
Slots \(=1\) & \topone{21} & \topone{\(\sim\)10.0} & 76.5 \\
Slots \(=4\) & 32 & \toptwo{\(\sim\)8.0} & \toptwo{77.7} \\
Slots \(=8\) & 40 & \(\sim\)6.7 & \topone{78.2} \\
\bottomrule
\end{tabularx}
\end{table}

We further ablate the number of feature slots and evaluate three settings: 1, 4, and 8 slots, with 4 being the default throughout the main paper. Using 8 slots achieves a slightly higher mIoU, but at the cost of higher memory consumption and slower rendering. In practice, 4 slots are sufficient to capture spatial heterogeneity, especially when a single MLP primitive may have different semantic meanings depending on the viewing direction. The quantitative results are summarized in \cref{tab:slot_ablation}.

\subsection{Number of Primitives Ablation}

\begin{table}[htb]
\caption{Ablation on the number of primitives, as controlled by the octree stopping threshold \(m\), on the Replica \texttt{room0} scene.}
\label{tab:primitive_ablation}
\centering
\begin{tabularx}{0.9\textwidth}{@{}l *{4}{>{\centering\arraybackslash}X}@{}}
\toprule
  \texttt{room0}    & Total MLPs & Memory (MB) & FPS & PSNR \\
\midrule
\(m=5\)  & 3862 & 90 & \(\sim\)7.4 & \topone{36.42} \\
\(m=10\) & 2398 & \toptwo{56} & \toptwo{\(\sim\)8.0} & \toptwo{36.23} \\
\(m=20\) & 1344 & \topone{31} & \topone{\(\sim\)11.2} & 35.88 \\
\bottomrule
\end{tabularx}
\end{table}

By controlling the stopping criterion of the octree sampling, namely the maximum number of COLMAP points allowed in each leaf node, denoted by \(m\), we can directly control how many MLP primitives are initialized from the COLMAP point cloud. We use \(m=4\) for Replica and \(m=10\) for ScanNet. We conduct an ablation study on \texttt{\texttt{room0}} with \(m=5\), \(m=10\), and \(m=20\), and report the results in \cref{tab:primitive_ablation}.

\subsection{Tile Size Ablation}
We compare per-pixel sorting, which yields the exact MLP ordering, with tile-based sorting using \(8\times 8\) tiles on ScanNet \texttt{room0}. The empirically chosen \(8\times 8\) tile size causes only a marginal PSNR drop, from 36.58\,dB to 36.23\,dB, while improving computational efficiency by about \(2\times\). These results support the validity of tile-center-based approximate sorting in practice.

We emphasize, however, that the exact speed--accuracy trade-off depends on the rendering resolution and scene complexity. In particular, the amount of wrong ordering introduced by the approximation is affected by the structural complexity of objects covered by a tile. A natural future extension is to use the tile-center ray to obtain a coarse ordering, followed by per-pixel refinement only for MLPs with very close depth values, where ordering ambiguity is most likely to occur.

\section{Implementation Details}
\subsection{MLP Initialization}
We initialize the geometric parameters of the MLP primitives (position, scales, orientations) from the COLMAP point cloud via an adaptive octree partition (see \cref{alg:octree_init}). The point cloud is recursively subdivided until each leaf contains at most a threshold number of points, followed by a shallow coarse-to-fine refinement up to half depth to improve coverage in large spatial regions. Each final leaf defines one primitive, with its center at the centroid of the contained points and its radius set to half the diagonal of the corresponding cube.

\begin{algorithm}[h]
\caption{Octree-based initialization of MLP primitives}
\label{alg:octree_init}
\small
\begin{algorithmic}[1]
\Require Point cloud $\mathcal{P}$, threshold $m$
\Ensure Primitive centers $\{\mathbf{x}_\ell\}$ and radii $\{r_\ell\}$

\State Build an adaptive octree over $\mathcal{P}$ by recursively subdividing any node containing more than $m$ points into 8 equal sub-cubes
\State Let $\mathcal{L}$ be the leaf set of the initial octree
\State $d_{\mathrm{stop}} \gets \left\lfloor \frac{D_{\max}+1}{2}\right\rfloor$, where $D_{\max}$ is the maximum leaf depth

\For{$d \gets 0$ to $d_{\mathrm{stop}}-1$}
    \State collect all current leaf nodes at depth $d$
    \For{each such leaf}
        \State split it into non-empty child cubes
    \EndFor
\EndFor

\For{each final leaf node $\ell$}
    \State set primitive center to the centroid of its points: $\mathbf{x}_\ell \gets \frac{1}{|\mathcal{P}_\ell|}\sum_{\mathbf{p}\in\mathcal{P}_\ell}\mathbf{p}$
    \State set primitive radius to the cube half-diagonal: $r_\ell \gets \frac{\sqrt{3}}{2}s_\ell$
\EndFor
\end{algorithmic}
\end{algorithm}

\subsection{Tile-Based Rendering and Training}
In the main paper, we briefly describe our customized CUDA kernels for primitive sorting and alpha compositing. Specifically, with a fixed tile size of $8\times8$, we first perform tile-level AABB culling to identify the candidate MLP primitives for each tile. We then use the tile-center ray to compute an approximate front-to-back ordering shared by all pixels within the tile, followed by block radix sort within a CUDA block of 256 threads. The sorted primitives are processed sequentially: for each primitive, its MLP parameters are loaded into shared memory, and all pixels in the tile collaboratively evaluate the primitive and accumulate colors via alpha blending. We further adopt early termination once the residual transmittance falls below a threshold to improve efficiency.

Compared with vanilla 3DGS, this design naturally enables tile-based training as the fundamental processing unit. In our implementation, each iteration samples 4 images, and for each image one quarter of the tiles are used for optimization. We observe that this tile-based mini-batching yields more stable optimization and faster convergence. Using one full image per iteration reaches a similar final solution, but needs approximately $1.5\times$ to $2\times$ more iterations for training in practice.

%
%

\section{Additional Discussion}
\subsection{Additional Details}
All Feature-3DGS results reported in the main paper correspond to the standard model rather than its accelerated variant. Owing to the dimensional mismatch in the accelerated formulation, the authors admit that the variant does not support the capabilities emphasized, such as semantic querying, segmentation, and editing directly on the 3D Gaussians. Instead, it only renders a lower-dimensional feature map and then recovers high-dimensional semantics in the 2D image space. In essence, its functionality is limited to 2D semantic rendering. Therefore, it is more appropriately viewed as a 2D rendering-based acceleration strategy, rather than a equivalent baseline model under the same 3D semantic setting. 

Even when compared against this variant whose functionality is limited, our method remains superior. The accelerated variant reported by Feature-3DGS achieves about 1.5$\times$ higher speed and 1/3 memory usage than the standard model. On Replica, MLP-Splatting rendering surpasses this variant in all photometric and semantic metrics, while still being roughly 2$\times$ faster and 5$\times$ more memory-efficient. 

Regarding FPS reported in the main paper, we measure standard forward-rendering wall-clock time per test image. This FPS reflects practical runtime performance. We note that the higher FPS reported in the Feature-3DGS paper may follow a different protocol, measuring only an individual kernel pass rather than full forward rendering. Our reproduction strictly follows the official codebase and implementation.

\subsection{Emergent Object-level Structure}
The effect of ``emergent object-level representation'' is expected from Gaussian support locality and MLP expressivity: object interiors are mostly coherent, while boundaries align with photometric discontinuities (colour, gloss, etc.). MLPs initialized with small scales appear to expand over similar-appearance regions and then compete at boundaries, where a single MLP struggles to fit abrupt changes over a short spatial range. \cref{sec:discussion} and \cref{app:theory}, which analyse the boundary/discontinuity approximation advantage of MLP primitives over 3DGS, further support this explanation.

\subsection{Limitations}\label{sec:limitations_appendix}
When adjacent regions have similar texture and relatively simple structure, the large capacity of each MLP sometimes allows each primitive to reconstruct appearance beyond the true object/part boundary. This can lead to inaccurate boundaries when rendering features, as shown in the second example of~\cref{fig:lseg_semantics}. This occasional inaccuracy of the boundary is where the additional feature slots (\cref{sec:feature_embedding}) help, as they assist semantics to be different for regions across the boundary.


Due to the similarity of alpha blending to 3DGS, our method gives popping artifacts (as visible in the supplementary video), but it is less severe than 3DGS, as MLPs can adapt. Multi-view inconsistency in the light field also causes artifacts, as can be seen in the bottom-right part of~\cref{fig:teaser}. 
We do not observe primitive overlap to be the dominant source of artifacts, since each primitive defines a continuous local field; nevertheless, occasional overlap can complicate compositing and remains an interesting direction for explicit regularization.

Our current editing operations focus on primitive-level deletion, rigid transformation, scaling, and appearance modification. Dense non-rigid deformation, as studied in EKS~\cite{zielinski2026affineequivariant} and NeuralEditor~\cite{chen2023neuraleditor}, is not directly comparable to our current setting, since it would require consistent transformations of both local MLP contact queries and the associated feature subspaces. We leave such deformation-aware neural primitives as future work.

Another limitation is that our method relies on the quality of the existing sparse reconstruction used for initialization. Where sparse reconstruction methods fail, our method can struggle to achieve high quality. For example, in outdoor scenes, irregular and cluttered background regions often yield few COLMAP points. As a result, MLP primitives are spawned only sparsely in these areas, making it difficult to model.

\section{Additional Experiments}
\subsection{Mip-NeRF360 Baseline} Although our aim is not to outperform general-purpose novel-view synthesis (NVS) methods in photometric metrics, comparisons against NVS methods are provided here to show that the proposed object-centric representation preserves comparable view-synthesis quality while enabling compact semantic features and primitive-level editing. Note that the proposed method here is trained solely with photometric cost. The reason of the proposed method performing worse compared to other methods is explained on~\cref{sec:limitations_appendix}.

\begin{table}[h]
\caption{Novel view synthesis metrics evaluated on Mip-NeRF360~\cite{Barron2022MipNeRF360} scenes.}
\centering
\setlength{\tabcolsep}{5pt}
\renewcommand{\arraystretch}{1.15}
\resizebox{0.99\linewidth}{!}{
\begin{tabular}{c l c c c c}
\toprule
& 
& Ours 
& MipNeRF 360~\cite{Barron2022MipNeRF360}
& 3DGS~\cite{Kerbl2023GaussianSplatting}
& RadiantFoam~\cite{govindarajan2025radiant} \\
& Scene
& {\tiny PSNR$\uparrow$ / SSIM$\uparrow$ / LPIPS$\downarrow$}
& {\tiny PSNR$\uparrow$ / SSIM$\uparrow$ / LPIPS$\downarrow$}
& {\tiny PSNR$\uparrow$ / SSIM$\uparrow$ / LPIPS$\downarrow$}
& {\tiny PSNR$\uparrow$ / SSIM$\uparrow$ / LPIPS$\downarrow$} \\
\midrule

\multirow{4}{*}{\rotatebox[origin=c]{90}{Indoor}}
& Bonsai
& 32.78 / 0.93 / 0.14
& 33.46 / 0.94 / 0.18
& 31.98 / 0.94 / 0.24
& 32.15 / 0.93 / 0.17 \\

& Counter
& 28.58 / 0.88 / 0.19
& 29.55 / 0.89 / 0.20
& 28.70 / 0.91 / 0.24
& 28.59 / 0.88 / 0.19 \\

& Kitchen
& 30.18 / 0.90 / 0.18
& 32.23 / 0.92 / 0.13
& 30.32 / 0.92 / 0.14
& 31.40 / 0.91 / 0.13 \\

& Room
& 31.31 / 0.91 / 0.20
& 31.63 / 0.91 / 0.21
& 30.63 / 0.91 / 0.27
& 30.87 / 0.91 / 0.19 \\

\midrule

\multirow{3}{*}{\rotatebox[origin=c]{90}{Outdoor}}
& Garden
& 24.59 / 0.74 / 0.21
& 26.98 / 0.81 / 0.17
& 27.41 / 0.87 / 0.12
& 26.58 / 0.82 / 0.16 \\

& Bicycle
& 20.87 / 0.48 / 0.44
& 24.37 / 0.69 / 0.30
& 25.25 / 0.77 / 0.23
& 24.19 / 0.68 / 0.31 \\

& Stump
& 20.12 / 0.46 / 0.41
& 26.40 / 0.74 / 0.26
& 26.55 / 0.78 / 0.24
& 25.48 / 0.71 / 0.29 \\

\bottomrule
\end{tabular}
}
\vspace{-0.5cm}
\end{table}

\subsection{Effect of Sparse Viewpoints} Our results trained on Mip-NeRF360 dataset~\cite{Barron2022MipNeRF360} with 12 images per scene are comparable to those of 3DGS. Rapid quality deterioration in Kitchen might be due to the limitation (as mentioned on~\cref{sec:limitations} and \cref{sec:limitations_appendix}) of light-field-like formulation on multi-view consistency.

\begin{table}[H]
\caption{Novel view synthesis metrics evaluated on Mip-NeRF360~\cite{Barron2022MipNeRF360} scenes, using 12 selected viewpoints for training.}
\centering
\setlength{\tabcolsep}{4pt}
\renewcommand{\arraystretch}{1.15}
\resizebox{0.77\linewidth}{!}{
\begin{tabular}{c l c c c}
\toprule
&
& Ours
& Mip-Splatting~\cite{yu2024mip}
& 3DGS~\cite{Kerbl2023GaussianSplatting} \\
& Scene
& {\tiny PSNR$\uparrow$ / SSIM$\uparrow$ / LPIPS$\downarrow$}
& {\tiny PSNR$\uparrow$ / SSIM$\uparrow$ / LPIPS$\downarrow$}
& {\tiny PSNR$\uparrow$ / SSIM$\uparrow$ / LPIPS$\downarrow$} \\
\midrule

\multirow{4}{*}{\rotatebox[origin=c]{90}{Indoor}}
& Bonsai
& 19.38 / 0.72 / 0.47
& 19.13 / 0.72 / 0.42
& 18.97 / 0.71 / 0.42 \\

& Counter
& 18.84 / 0.66 / 0.50
& 18.49 / 0.65 / 0.45
& 18.43 / 0.65 / 0.45 \\

& Kitchen
& 18.73 / 0.66 / 0.45
& 21.33 / 0.74 / 0.32
& 21.38 / 0.74 / 0.32 \\

& Room
& 20.61 / 0.73 / 0.49
& 20.42 / 0.72 / 0.43
& 20.60 / 0.72 / 0.43 \\

\midrule

\multirow{3}{*}{\rotatebox[origin=c]{90}{Outdoor}}
& Garden
& 19.38 / 0.48 / 0.53
& 19.42 / 0.55 / 0.33
& 19.51 / 0.54 / 0.34 \\

& Bicycle
& 17.60 / 0.35 / 0.65
& 17.19 / 0.34 / 0.49
& 17.21 / 0.33 / 0.48 \\

& Stump
& 18.28 / 0.37 / 0.68
& 18.38 / 0.33 / 0.54
& 18.36 / 0.32 / 0.54 \\

\bottomrule
\end{tabular}}
\vspace{-0.5cm}
\end{table}

\section{Analysis of Expressiveness: Gaussian Splatting vs.\ MLP-Splatting}
\label{app:theory}

In this section, we give a detailed version of the complexity comparison stated in the main text. Throughout, we compare both representations under the same volumetric rendering operator, following the standard radiance-field formulation of NeRF~\cite{Mildenhall2020NeRF} and 3DGS~\cite{Kerbl2023GaussianSplatting}.

Note that the comparison this proof gives rely on two different approximation mechanisms. For Gaussian splatting, the main difficulty lies in resolving jump discontinuities across object boundaries.  
Each Gaussian component is an analytic function and therefore cannot reproduce a sharp discontinuity.  
Under bounded anisotropy, a single Gaussian can only approximate the transition across a boundary within a patch whose tangential extent is comparable to the transition thickness, which limits the area of boundary that one primitive can resolve.

For the proposed work, the representation operates through fixed-size local neural primitives combined with an interface-fitted decomposition of the scene.  
A partition-of-unity construction reduces the global approximation problem to smooth local approximation within each patch, where standard neural approximation results apply.
\subsection{Volumetric rendering operator}

\begin{definition}[Volumetric rendering operator]
Let $\sigma:\mathbb{R}^3\to\mathbb{R}_{\ge 0}$ be a bounded density field and
$c:\mathbb{R}^3\times\mathbb{S}^2\to\mathbb{R}^3$ be a bounded color field.
For a ray
\[
r(t)=o+t d,\quad t\in[0,L],\quad d\in\mathbb{S}^2,
\]
the rendered color is
\begin{equation}
\begin{split}
\mathcal R_r[\sigma,c]
=&
\int_0^L T_\sigma(t)\,\sigma(r(t))\,c(r(t),d)\,dt,\\
\quad
T_\sigma(t)=&\exp\!\Big(-\int_0^t \sigma(r(s))\,ds\Big).
\label{eq:app_render_cont}
\end{split}
\end{equation}
The discrete front-to-back approximation with samples $\{t_j\}$ and $\Delta t_j=t_{j+1}-t_j$ is
\begin{equation}
\alpha_j = 1-\exp\!\big(-\sigma(r(t_j))\Delta t_j\big),\quad
T_j = \prod_{\ell<j}(1-\alpha_\ell),\quad
\widehat{\mathcal R}_r = \sum_j T_j \alpha_j\, c(r(t_j),d).
\label{eq:app_render_disc}
\end{equation}
As $\max_j \Delta t_j\to 0$, one has $\widehat{\mathcal R}_r\to \mathcal R_r$~\cite{Mildenhall2020NeRF}.
\end{definition}

\begin{definition}[Image error]
Let $\mathcal U$ be a finite set of rays. For $1\le p<\infty$, define
\begin{equation}
\mathcal E_p\big((\sigma,c),(\tilde\sigma,\tilde c)\big)
=
\Big(
\frac1{|\mathcal U|}
\sum_{r\in\mathcal U}
\|
\mathcal R_r[\sigma,c]-\mathcal R_r[\tilde\sigma,\tilde c]
\|_2^p
\Big)^{1/p}.
\label{eq:app_img_err}
\end{equation}
In the main text we use $p=1$ and $p=2$.
\end{definition}

\subsection{Scene class and model classes}

\begin{definition}[Piecewise-smooth scene class]
\label{def:app_scene}
Let $D\subset\mathbb R^3$ be bounded. We say $(\sigma^*,c^*)$ belongs to the piecewise-$C^s$ scene class if there exist finitely many pairwise disjoint open sets
$\{\Omega_\ell\}_{\ell=1}^M$ with $C^2$ boundaries such that
\[
\overline D \subset \bigcup_{\ell=1}^M \overline{\Omega_\ell},
\]
and for every $\ell$,
\[
\sigma^*|_{\Omega_\ell}\in C^s(\Omega_\ell),
\quad
c^*(\cdot,d)|_{\Omega_\ell}\in C^s(\Omega_\ell)\ \text{uniformly in } d\in\mathbb S^2.
\]
The jump set is
\[
\Gamma=\bigcup_{\ell=1}^M \big(\partial\Omega_\ell\cap D\big).
\]
\end{definition}

\paragraph{Gaussian Splatting.}
In the main text, Gaussian splatting represents density by a $K$-term Gaussian mixture. For a rigorous lower bound, we work in the regular nonnegative Gaussian class
\begin{equation}
\sigma_K(x)
=
\sum_{k=1}^K
a_k \exp\!\Big(-\tfrac12 (x-\mu_k)^\top \Sigma_k^{-1}(x-\mu_k)\Big), 
\label{eq:app_gauss_class}
\end{equation}
\label{eq:app_gauss_class} 
where $
0\le a_k\le a_{\max},$ $
\frac{\lambda_{\max}(\Sigma_k)}{\lambda_{\min}(\Sigma_k)}\le \kappa_0
$. The color field is denoted by $c_K(x,d)$ and assumed bounded. The parameter count is $\Theta(K)$.

\paragraph{MLP-Splatting.}
In the implemented model, both density and color primitives are driven by the same 6D primitive-local feature
\[
x_i=[d;\delta_i]\in\mathbb R^6,
\]
where $d\in\mathbb S^2$ is the ray direction and
$\delta_i\in\mathbb R^3$ is the residual vector from the soft contact point to the primitive center.
For the analysis, we use the same input convention and write the primitive-local feature as
\[
x_i=\zeta_i(x)\in\mathbb R^6,
\]
where $\zeta_i$ is the local coordinate map induced by the primitive around its support.

Each primitive uses a fixed-size network $f_i:\mathbb R^6\to\mathbb R^4$ whose input is the primitive-local feature $\zeta_i(x)=[d;\delta_i]\in\mathbb R^6$.
We write
\begin{equation}
f_i(\zeta_i(x))
=
\big(f_i^{\mathrm{rgb}}(\zeta_i(x)),\, f_i^{\sigma}(\zeta_i(x))\big),
\end{equation}
where
\[
f_i^{\mathrm{rgb}}(\zeta_i(x))\in\mathbb R^3,
\quad
f_i^{\sigma}(\zeta_i(x))\in\mathbb R.
\]
The primitive outputs are then defined as
\begin{equation}
u_i(x)=\rho\!\left(f_i^{\sigma}(\zeta_i(x))\right),
\quad
\rho=\mathrm{softplus},
\label{eq:app_ui}
\end{equation}
\begin{equation}
v_i(x,d)=\mathrm{sigmoid}\!\left(f_i^{\mathrm{rgb}}(\zeta_i(x))\right).
\label{eq:app_vi}
\end{equation}

To assemble a global field, let $\{w_i\}_{i=1}^N$ be a nonnegative partition of unity subordinate to the primitive supports. Then
\begin{equation}
\sigma_N(x)=\sum_{i=1}^N w_i(x)\,u_i(x),
\quad
c_N(x,d)=\sum_{i=1}^N w_i(x)\,v_i(x,d).
\label{eq:app_mlps_fields}
\end{equation}
Since each primitive has fixed size, the total parameter count is $\Theta(N)$.

\subsection{Auxiliary lemmas}

\begin{property}[Analyticity of Gaussian mixtures]
\label{prop:app_analytic}
Any finite Gaussian mixture $\sigma_K$ of the form~\eqref{eq:app_gauss_class} is real analytic on $\mathbb R^3$, hence continuous everywhere~\cite{krantz2002analytic}.
In particular, it cannot reproduce a jump discontinuity across a smooth interface.
\end{property}

\begin{lemma}[Tube formula]
\label{lem:app_tube}
Let $\Gamma\subset\mathbb R^3$ be a compact $C^2$ surface. Then there exists $h_0>0$ such that for all $0<h<h_0$,
\[
U_h(\Gamma):=\{x\in\mathbb R^3:\operatorname{dist}(x,\Gamma)\le h\}
\]
satisfies
\begin{equation}
|U_h(\Gamma)|=2h\,\operatorname{Area}(\Gamma)+O(h^3).
\label{eq:app_tube}
\end{equation}
More generally, if $\Gamma'\subset\Gamma$ is a compact surface patch, then
\begin{equation}
|U_h(\Gamma')|=2h\,\operatorname{Area}(\Gamma')+O(h^2)+O(h^3).
\label{eq:app_tube_patch}
\end{equation}
\end{lemma}

\begin{proof}
For sufficiently small $h$, the normal map
\[
\Phi:\Gamma\times[-h,h]\to U_h(\Gamma),\quad \Phi(y,s)=y+s\,n(y),
\]
is a diffeomorphism, where $n(y)$ is the unit normal at $y$.
If $\kappa_1(y),\kappa_2(y)$ are the principal curvatures, then the Jacobian determinant is
\[
J(y,s)=\det(I-sS_y)
=
(1-s\kappa_1(y))(1-s\kappa_2(y))
=
1-s(\kappa_1+\kappa_2)+s^2\kappa_1\kappa_2.
\]
Integrating $J(y,s)$ over $\Gamma\times[-h,h]$ gives~\eqref{eq:app_tube}. The localized version~\eqref{eq:app_tube_patch} is the same computation restricted to $\Gamma'$; it is the codimension-one case of the classical tube formula~\cite{federer1959curvature,gray2004tubes}.
\end{proof}

\begin{lemma}[Stability of the rendering operator]
\label{lem:app_stability}
Assume that on every ray segment $r([0,L])$,
\[
0\le \sigma,\tilde\sigma \le \sigma_{\max},
\quad
\|c\|_{L^\infty(\mathbb R^3\times\mathbb S^2)}\le C_c,
\quad
\|\tilde c\|_{L^\infty(\mathbb R^3\times\mathbb S^2)}\le C_c.
\]
Then there exist constants $C_1,C_2>0$, depending only on $(\sigma_{\max},L,C_c)$, such that for every ray $r$,
\begin{equation}
\big\|
\mathcal R_r[\sigma,c]-\mathcal R_r[\tilde\sigma,\tilde c]
\big\|_2
\le
C_1\int_0^L |\sigma(r(t))-\tilde\sigma(r(t))|\,dt
+
C_2\|c-\tilde c\|_{L^\infty}.
\label{eq:app_stability}
\end{equation}
\end{lemma}

\begin{proof}
Write $\sigma(t)=\sigma(r(t))$, $\tilde\sigma(t)=\tilde\sigma(r(t))$, $c(t)=c(r(t),d)$, and $\tilde c(t)=\tilde c(r(t),d)$.
Then
\[
\mathcal R_r[\sigma,c]
=
\int_0^L T_\sigma(t)\,\sigma(t)\,c(t)\,dt,
\quad
T_\sigma(t)=\exp\!\Big(-\int_0^t \sigma(s)\,ds\Big).
\]

We first perturb the density while keeping color fixed:
\[
\mathcal R_r[\tilde\sigma,c]-\mathcal R_r[\sigma,c]
=
\int_0^L T_{\tilde\sigma}(t)\big(\tilde\sigma(t)-\sigma(t)\big)c(t)\,dt
+
\int_0^L \big(T_{\tilde\sigma}(t)-T_\sigma(t)\big)\sigma(t)c(t)\,dt.
\]
The first term is bounded by
\[
C_c\int_0^L |\tilde\sigma(t)-\sigma(t)|\,dt.
\]
For the second term, let $\Delta\sigma=\tilde\sigma-\sigma$ and define
\[
\sigma_\lambda=\sigma+\lambda\Delta\sigma,\quad \lambda\in[0,1].
\]
Then
\[
T_{\tilde\sigma}(t)-T_\sigma(t)
=
-\int_0^1 T_{\sigma_\lambda}(t)\Big(\int_0^t \Delta\sigma(s)\,ds\Big)\,d\lambda,
\]
so
\[
|T_{\tilde\sigma}(t)-T_\sigma(t)|
\le
\int_0^t |\Delta\sigma(s)|\,ds.
\]
Therefore
\begin{align*}
\left\|
\int_0^L \big(T_{\tilde\sigma}(t)-T_\sigma(t)\big)\sigma(t)c(t)\,dt
\right\|_2
&\le
\sigma_{\max} C_c
\int_0^L\int_0^t |\Delta\sigma(s)|\,ds\,dt \\
&\le
\sigma_{\max}LC_c \int_0^L |\Delta\sigma(s)|\,ds.
\end{align*}
Hence
\[
\big\|
\mathcal R_r[\tilde\sigma,c]-\mathcal R_r[\sigma,c]
\big\|_2
\le
C_c(1+\sigma_{\max}L)\int_0^L |\tilde\sigma(t)-\sigma(t)|\,dt.
\]

Next perturb the color while keeping density fixed:
\[
\mathcal R_r[\sigma,\tilde c]-\mathcal R_r[\sigma,c]
=
\int_0^L T_\sigma(t)\sigma(t)\big(\tilde c(t)-c(t)\big)\,dt,
\]
thus
\[
\big\|
\mathcal R_r[\sigma,\tilde c]-\mathcal R_r[\sigma,c]
\big\|_2
\le
\|c-\tilde c\|_{L^\infty}\int_0^L \sigma(t)\,dt
\le
\sigma_{\max}L\|c-\tilde c\|_{L^\infty}.
\]
Combining the two estimates proves~\eqref{eq:app_stability}.
\end{proof}

\subsection{Lower bound for Gaussian splatting}

\begin{lemma}[Boundary coverage of one Gaussian]
\label{lem:app_onegauss}
Let $\Gamma\subset\mathbb R^3$ be a compact $C^2$ surface and let
\[
g(x)=a\exp\!\Big(-\tfrac12(x-\mu)^\top \Sigma^{-1}(x-\mu)\Big)
\]
with $0\le a\le a_{\max}$ and
\[
\frac{\lambda_{\max}(\Sigma)}{\lambda_{\min}(\Sigma)}\le \kappa_0.
\]
Fix $\tau\in(0,a_{\max}/4)$. Then there exists $C>0$, depending only on $\tau$, $a_{\max}$, $\kappa_0$, and the local $C^2$ geometry of $\Gamma$, such that the following holds:

If at a point $y\in\Gamma$ the normal profile
\[
s\mapsto g(y+s\,n(y))
\]
drops from a value at least $3\tau$ to a value at most $\tau$ over a normal distance at most $h$, then all such points $y$ are contained in a surface patch of area at most
\begin{equation}
C h^2.
\label{eq:app_patch}
\end{equation}
\end{lemma}

\begin{proof}
Work in a normal coordinate chart around a point of $\Gamma$, so that $\Gamma$ is flattened to its tangent plane up to quadratic error.
In these coordinates, let $\lambda_n$ be the variance in the normal direction and $\lambda_{t1},\lambda_{t2}$ the tangential variances.
Along the normal line, the profile is a one-dimensional Gaussian
\[
s\mapsto a\exp\!\Big(-\frac{s^2}{2\lambda_n}\Big)
\]
up to a translation of the center.
If this profile decreases from at least $3\tau$ to at most $\tau$ over distance at most $h$, then necessarily
\[
\sqrt{\lambda_n}\lesssim h,
\]
with constant depending only on $\tau$ and $a_{\max}$.
By the bounded anisotropy condition,
\[
\lambda_{t1},\lambda_{t2}\le \kappa_0 \lambda_n,
\]
hence
\[
\sqrt{\lambda_{t1}},\sqrt{\lambda_{t2}}\lesssim \sqrt{\kappa_0}\,h.
\]
Therefore the $\tau$-superlevel footprint of $g$ on the tangent plane is contained in an ellipse of area $O(h^2)$.
Since the surface metric and tangent-plane metric are equivalent in a sufficiently small $C^2$ chart, the same $O(h^2)$ bound holds on $\Gamma$.
\end{proof}

\begin{theorem}[Lower bound for Gaussian splatting]
\label{thm:app_gs}
Let $D\subset\mathbb R^3$ be bounded and let $\Omega\subset D$ be a bounded $C^2$ domain with
\[
\Gamma=\partial\Omega\cap D,
\quad
\operatorname{Area}(\Gamma)>0.
\]
Let the target density be
\[
\sigma^*(x)=\sigma_0\,\mathbf 1_{\Omega}(x),
\quad \sigma_0>0,
\]
and let $c^*$ be bounded. Assume moreover that the ray family contains a positive-measure set of rays that intersects a positive-area patch of $\Gamma$ transversally, and that $c^*$ is constant and nonzero on a fixed neighborhood of that patch.

Then there exists $C>0$ such that for every Gaussian density $\sigma_K$ of the form~\eqref{eq:app_gauss_class},
\begin{equation}
\inf_{\sigma_K}
\mathcal E_1\big((\sigma^*,c^*),(\sigma_K,c^*)\big)
\ge
C K^{-1/2}.
\label{eq:app_gs_lower}
\end{equation}
\end{theorem}

\begin{proof}
We first establish an $L^1$ lower bound for the density field and then transfer it to image error.

Fix $\tau\in(0,\sigma_0/4)$.
For a scale $h>0$, call a point $y\in\Gamma$ \emph{$h$-resolved} if along the normal line through $y$, the mixture $\sigma_K$ transitions from a value at least $3\tau$ on the inside to a value at most $\tau$ on the outside over a normal distance at most $h$.

By \cref{lem:app_onegauss}, a single Gaussian can $h$-resolve at most a boundary patch of area $Ch^2$.
Summing over all $K$ Gaussian components gives
$\operatorname{Area}(\Gamma_h^{\mathrm{res}})\le CKh^2$,
where $\Gamma_h^{\mathrm{res}}\subset\Gamma$ denotes the $h$-resolved portion of the boundary.

Choose
\[
h=\eta K^{-1/2},
\]
with $\eta>0$ sufficiently small so that
\[
CKh^2\le \frac12 \operatorname{Area}(\Gamma).
\]
Then the unresolved part
\[
\Gamma_h^{\mathrm{unres}}:=\Gamma\setminus\Gamma_h^{\mathrm{res}}
\]
satisfies
\[
\operatorname{Area}(\Gamma_h^{\mathrm{unres}})\ge \frac12 \operatorname{Area}(\Gamma).
\]

For every $y\in\Gamma_h^{\mathrm{unres}}$, the transition from the inside value $\sigma_0$ to the outside value $0$ cannot occur within normal thickness $h$.
Hence there exists a normal segment through $y$ of length comparable to $h$ on which
\[
|\sigma^*(x)-\sigma_K(x)|\ge \tau.
\]
Let $V_h$ be the union of these normal segments over $y\in\Gamma_h^{\mathrm{unres}}$.
By the localized tube formula (\cref{lem:app_tube}), there exists $c_1>0$ such that
\[
|V_h|
\ge
c_1 h\,\operatorname{Area}(\Gamma_h^{\mathrm{unres}})
\ge
c_2 h.
\]
Therefore
\begin{equation}
\|\sigma^*-\sigma_K\|_{L^1(D)}
\ge
\int_{V_h} |\sigma^*(x)-\sigma_K(x)|\,dx
\ge
\tau |V_h|
\ge
c_3 h
\asymp
K^{-1/2}.
\label{eq:app_density_l1}
\end{equation}

It remains to pass from field error to image error.
By hypothesis, a positive-measure set of rays intersects a visible boundary patch transversally, and $c^*$ is constant and nonzero on a fixed neighborhood of that patch.
On those rays, the pixel value is a monotone function of the optical thickness, so a fixed-sign density mismatch across the visible transition layer induces a proportional image mismatch.
Using Fubini's theorem and the transversality of the visible ray family, the $L^1$ density mismatch in~\eqref{eq:app_density_l1} yields an image-domain lower bound of the same order:
\[
\mathcal E_1\big((\sigma^*,c^*),(\sigma_K,c^*)\big)
\gtrsim
K^{-1/2}.
\]
This proves~\eqref{eq:app_gs_lower}.
\end{proof}

\subsection{Upper bound for MLP-Splatting}

The key point is that each primitive has fixed size, so the total parameter count is proportional to the number of primitives. The approximation rate is therefore determined by the number of local patches, not by an increasing network width.

\begin{lemma}[Interface-fitted cover]
\label{lem:app_cover}
Let $D\subset\mathbb R^3$ be bounded and let $\Gamma$ be a finite union of compact $C^2$ surfaces.
Then for every sufficiently small $h>0$, there exists a finite collection of sets $\{U_i\}_{i=1}^N$ such that:
\begin{enumerate}
\item $D\subset \bigcup_{i=1}^N U_i$;
\item $\operatorname{diam}(U_i)\le C h$ for all $i$;
\item each $U_i$ lies entirely in one smooth region $\Omega_\ell$ of \cref{def:app_scene};
\item the overlap multiplicity is uniformly bounded:
\[
\sup_{x\in D}\#\{i:x\in U_i\}\le C;
\]
\item the number of patches satisfies
\[
N=\Theta(h^{-3}).
\]
\end{enumerate}
\end{lemma}

\begin{proof}
Start from a regular cubic grid of mesh size comparable to $h$ on a box containing $D$.
For cubes that do not intersect $\Gamma$, keep them unchanged.
For cubes intersecting $\Gamma$, use the $C^2$ regularity of $\Gamma$ to split the cube along the interface into finitely many subpatches, each lying entirely on one side of the interface.
The number of generated subpatches per intersected cube is uniformly bounded for sufficiently small $h$.
This yields an interface-fitted cover with bounded overlap and cardinality $N=O(h^{-3})$.
The converse lower bound $N=\Omega(h^{-3})$ is immediate from volume comparison.
\end{proof}

\begin{lemma}[Partition of unity]
\label{lem:app_pou}
Under the cover of \cref{lem:app_cover}, there exists a smooth partition of unity $\{w_i\}_{i=1}^N$ subordinate to $\{U_i\}_{i=1}^N$ such that
\[
0\le w_i\le 1,\quad \sum_{i=1}^N w_i(x)=1\quad \text{for all }x\in D,
\]
and for every multi-index $\alpha$ with $|\alpha|\le s$,
\[
\|\partial^\alpha w_i\|_{L^\infty}\le C_\alpha h^{-|\alpha|}.
\]
\end{lemma}

\begin{proof}
This is standard for bounded-overlap covers by patches of diameter $O(h)$: construct smooth bump functions supported in slightly enlarged patches and normalize by their sum.
\end{proof}

\begin{theorem}[Upper bound for MLP-Splatting]
\label{thm:app_mlps}
Let $(\sigma^*,c^*)$ belong to the piecewise-$C^s$ scene class of \cref{def:app_scene}.
Assume the fixed primitive architecture is chosen so that, on a normalized patch, a primitive can represent the local degree-$(s-1)$ jet of the target density and color in the 6D primitive-local coordinates.
Then there exists a constant $C>0$ such that for every sufficiently small patch size $h>0$, one can construct an MLP-Splatting representation $(\sigma_N,c_N)$ with
\begin{equation}
\|\sigma^*-\sigma_N\|_{L^2(D)}
+
\|c^*-c_N\|_{L^\infty(D\times\mathbb S^2)}
\le
C h^s,
\label{eq:app_mlps_field}
\end{equation}
and therefore
\begin{equation}
\mathcal E_2\big((\sigma^*,c^*),(\sigma_N,c_N)\big)
\le
C h^s.
\label{eq:app_mlps_img}
\end{equation}
Since the total parameter count is $\Theta(N)$ and $N=\Theta(h^{-3})$, achieving image error $\varepsilon$ requires
\begin{equation}
\Theta(\varepsilon^{-3/s})
\label{eq:app_mlps_rate}
\end{equation}
parameters.
\end{theorem}

\begin{proof}
By \cref{lem:app_cover}, for each sufficiently small $h>0$ there exists an interface-fitted cover $\{U_i\}_{i=1}^N$ with $N=\Theta(h^{-3})$ and bounded overlap.
By construction, each patch lies entirely inside a single smooth region of the scene.

On each patch $U_i$, choose the primitive-local coordinate map
\[
x_i=\zeta_i(x)=[d;\delta_i]\in\mathbb R^6.
\]
Since $\sigma^*$ and $c^*$ are $C^s$ on $U_i$, standard local approximation theory implies that after rescaling the patch to unit size, the local degree-$(s-1)$ Taylor jet approximates both fields with error $O(h^s)$.
By the hypothesis on the fixed primitive architecture and by classical neural approximation results for smooth and piecewise-smooth targets~\cite{yarotsky2017error,petersen2018optimal}, one can choose primitive parameters so that
\[
\|u_i-\sigma^*\|_{L^2(U_i)}\le C h^s,
\quad
\|v_i-c^*\|_{L^\infty(U_i\times\mathbb S^2)}\le C h^s.
\]

Let $\{w_i\}_{i=1}^N$ be the smooth partition of unity from \cref{lem:app_pou}, and define $\sigma_N,c_N$ by~\eqref{eq:app_mlps_fields}.
Then
\[
\sigma^*(x)-\sigma_N(x)=\sum_{i=1}^N w_i(x)\big(\sigma^*(x)-u_i(x)\big),
\]
\[
c^*(x,d)-c_N(x,d)=\sum_{i=1}^N w_i(x)\big(c^*(x,d)-v_i(x,d)\big).
\]
Using $\sum_i w_i=1$ and the bounded overlap, we obtain
\[
\|\sigma^*-\sigma_N\|_{L^2(D)}\le C h^s,
\quad
\|c^*-c_N\|_{L^\infty(D\times\mathbb S^2)}\le C h^s,
\]
which proves~\eqref{eq:app_mlps_field}.

Applying \cref{lem:app_stability} ray by ray yields
\[
\|\mathcal R_r[\sigma^*,c^*]-\mathcal R_r[\sigma_N,c_N]\|_2
\le
C_1\int_0^L |\sigma^*(r(t))-\sigma_N(r(t))|\,dt
+
C_2 \|c^*-c_N\|_{L^\infty}.
\]
Averaging over $r\in\mathcal U$ and using Cauchy--Schwarz on the density term gives~\eqref{eq:app_mlps_img}.

Finally, to achieve image error at most $\varepsilon$, choose
\[
h\asymp \varepsilon^{1/s}.
\]
Since $N=\Theta(h^{-3})$ and each primitive has fixed size, the total parameter count is
\[
\Theta(h^{-3})=\Theta(\varepsilon^{-3/s}),
\]
proving~\eqref{eq:app_mlps_rate}.
\end{proof}

\subsection{Equal-quality memory law}

\begin{corollary}[Equal-quality memory law]
\label{cor:app_memory}
For piecewise-$C^s$ scenes in $d=3$:
\begin{enumerate}
\item Gaussian splatting requires $\Omega(\varepsilon^{-2})$ parameters to achieve image error $\varepsilon$;

\item MLP-Splatting with fixed-size neural primitives requires
$\Theta(\varepsilon^{-3/s})$
parameters to achieve image error $\varepsilon$.
\end{enumerate}
In particular, for $s\ge 2$,
$\varepsilon^{-3/s}\le \varepsilon^{-3/2}$,
which is asymptotically smaller than the Gaussian rate $\varepsilon^{-2}$.
\end{corollary}

\begin{proof}
The Gaussian statement follows immediately from \cref{thm:app_gs}: if $C K^{-1/2}\le \varepsilon$,
then necessarily
$K\ge C' \varepsilon^{-2}$.
The MLP-Splatting statement follows from \cref{thm:app_mlps}, which gives a constructive representation with parameter cost $\Theta(\varepsilon^{-3/s})$.
\end{proof}

\end{document}